\documentclass[sigconf]{acmart}

\AtBeginDocument{%
  \providecommand\BibTeX{{%
    \normalfont B\kern-0.5em{\scshape i\kern-0.25em b}\kern-0.8em\TeX}}}
\usepackage{booktabs} % For formal tables

\copyrightyear{2023}
\acmYear{2023}
\setcopyright{rightsretained}
\acmConference[KDD '23]{Proceedings of the 29th ACM SIGKDD Conference on Knowledge Discovery and Data Mining}{August 6--10, 2023}{Long Beach, CA, USA}
\acmBooktitle{Proceedings of the 29th ACM SIGKDD Conference on Knowledge Discovery and Data Mining (KDD '23), August 6--10, 2023, Long Beach, CA, USA}
\acmDOI{10.1145/3580305.3599512}
\acmISBN{979-8-4007-0103-0/23/08}

\makeatletter

\gdef\@copyrightpermission{
  \begin{minipage}{0.3\columnwidth}
   \href{}{\includegraphics[width=0.90\textwidth]{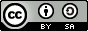}}
  \end{minipage}\hfill
  \begin{minipage}{0.7\columnwidth}
   \href{}{This work is licensed under a Creative Commons Attribution International 4.0 License.}
  \end{minipage}
  \vspace{5pt}
}

\usepackage{amscd,amsfonts,amsbsy,rotating}
\usepackage{makecell}
\usepackage{pifont}
\usepackage{balance}
\usepackage{graphicx}
\usepackage{epsfig,epstopdf}
\usepackage{subfigure}
\usepackage{multirow}
\usepackage{booktabs}
\usepackage{color,xcolor}
\usepackage{url}
\usepackage{latexsym,bm}
\usepackage{enumitem,balance,mathtools}
\usepackage{wrapfig}
\usepackage{euscript}
\usepackage{ifpdf}
\usepackage{diagbox}
\usepackage{caption}
\usepackage{makecell}
\usepackage{subfigure}
\usepackage[leftcaption]{sidecap}
\usepackage{textcomp}
\usepackage[normalem]{ulem}
\usepackage{breakurl}
\usepackage{microtype}
\usepackage{graphicx}
\usepackage{subfigure}
\usepackage{balance}
\usepackage{booktabs} % for professional tables
\usepackage{amsthm}
\usepackage{array}
\usepackage{graphicx}
\usepackage{clrscode}
\usepackage{subfigure}
\usepackage{multirow}
\usepackage{multicol}
\usepackage{float}
\usepackage{color}
\usepackage{xcolor}
\usepackage{amsopn}
\usepackage{mathrsfs}
\usepackage{mathtools}
\usepackage{amsmath}
\usepackage{booktabs}
\usepackage{arydshln}
\usepackage{hyperref}
\usepackage{blkarray}
\usepackage{enumerate}
\usepackage{courier}
\usepackage{mathrsfs}
\usepackage{rotating}
\usepackage{bm}
\usepackage{wrapfig}
\usepackage{subfigure}
\usepackage{array}
\usepackage{ragged2e}
\usepackage{amsmath}
\usepackage{threeparttable}  
\newcolumntype{N}{@{}m{0pt}@{}}

\theoremstyle{definition}
\newtheorem{definition}{Definition}
\theoremstyle{theorem}

\theoremstyle{assumption}
\newtheorem{assumption}{Assumption}

\theoremstyle{proof}

\theoremstyle{remark}
\newtheorem*{remark}{Remark}
\newcolumntype{g}{>{\columncolor{Gray}}c}
\usepackage[lined,boxruled,commentsnumbered]{algorithm2e}
\begin{document}
\title{Specify Robust Causal Representation from Mixed Observations}
\author{Mengyue Yang~}
 \affiliation{
  \institution{University College London}
  \city{London}
    \country{United Kingdom}
}
 \email{mengyue.yang.20@ucl.ac.uk}

\author{Xinyu Cai~}
 \affiliation{
  \institution{Nanyang Technological University}
  \city{Singapore}
    \country{Singapore}
}
 \email{xinyu.cai@ntu.edu.sg}

\author{Furui Liu~}
\authornote{Corresponding Author}
 \affiliation{
  \institution{Zhejiang Lab}
  \city{Hangzhou}
    \country{China}
}
 \email{liufurui@zhejianglab.com}

\author{Weinan Zhang~}
 \affiliation{
  \institution{Shanghai Jiao Tong University}
  \city{Shanghai}
    \country{China}
}
 \email{wnzhang@apex.sjtu.edu.cn}
 \author{Jun Wang~}
 \affiliation{
  \institution{University College London}
  \city{London}
    \country{United Kingdom}
}
 \email{j.wang@cs.ucl.ac.uk}

\renewcommand{\shortauthors}{Mengyue Yang, Xinyu Cai, Furui Liu, Weinan Zhang, \& Jun Wang}

% \jef{TODO LIST:
% \begin{itemize}
%     \item make sure that we emphasise that we are interested in the true relevance score. compared to standard fairness metric. We need to be clear about this so motivate why we are going with data debias over fairness constraint-based. The end goal might be the same i.e. same performance across different sensitive groups but we are guaranteeing different things. (done)
%     \item Make sure all the equations are fitted nicely and centered (done)
%     \item Make sure experiemental section is clearly written and especailly why we are using AUF-F etc (done)
%     \item Make sure we have a limitations sections where we say that future work would consider fairness constraints as well (done)
%     \item empphasis, what can our method do that constraint based cant necessarity do 
%     \item MENGYUE please add references (done)
%     \item MENGYUE please make sure that the vectors are bolded consistently (done)
%     \item MENGYUE related work section with fairness pipeline (done)
%     \item experimental section amking clear that we are only focused on the debiaseing part of the fairness pipeline (done)
%     \item Once we are done with the above we can send the paper to Kevin and start cutting parts out of the paper.
% \end{itemize}}

\begin{abstract}
% Learning representations from purely observations concerns the problem of learning a low-dimensional, compact representation which is beneficial to  prediction models. 
Learning representations purely from observations concerns the problem of learning a low-dimensional, compact representation which is beneficial to  prediction models.  Under the hypothesis that the intrinsic latent factors follow some casual generative models, we argue that by learning a causal representation, which is the minimal sufficient causes of the whole system, we can improve the robustness and generalization performance of machine learning models. In this paper, we develop a learning method to learn such representation from observational data by regularizing the learning procedure with mutual information measures, according to the hypothetical factored causal graph. We theoretically and empirically show that the models trained with the learned causal representations are more robust under adversarial attacks and distribution shifts compared with baselines. The supplementary materials are available at https://github.com/ymy4323460/CaRI/.

\end{abstract}
% \settopmatter{printacmref=false}



\maketitle
\section{Introduction}
Causal representation learning is an effective approach for extracting invariant, cross-domain stable causal information, which is believed to be able to improve sample efficiency by understanding the underlying generative mechanism from observational data \cite{scholkopf2021toward, DBLP:journals/advcs/AyP08}. Causal representation learning is widely applied in many real-world applications like recommendation systems, search engines etc.\cite{sun2018recurrent, okura2017embedding, zhang2017joint, shi2018heterogeneous}. Recently, multiple approaches were proposed to learn the invariant causal representations, which are supposed to encode underlying causal generative systems describing the data, based on the problem-specific priors. 

The usual theory to implement it is called Independent Causal Machine (ICM) \cite{parascandolo2018learning} principle, which can be applied to identify the cause information when all factors are observable. However, when the variables are unobservable in general and complex systems, this method usually does not work. Given that most methods employ a generative model,
% Standing by the perspective of generative models, 
the main reason for such failure is due to the observation data (e.g. human images) is entangled by causal variables. To tackle this problem, previous works learned latent representations to capture the causal properties, e.g., causal disentanglement methods \cite{yang2021causalvae,shen2020disentangled} and invariant causal representation learning method \cite{arjovsky2019invariant, lu2021nonlinear}. However, additional information like causal variable labels and domain information should be provided, which is usually unavailable in real-world systems.

In this paper, we aim at disentangling the causal variables from an information theoretical view without providing additional supervision signals. Supposing that the factors are casually structured, we formalize a causal system as in Fig.\ref{fig:intro} (a), which is commonly accepted by the causality community \cite{wang2021desiderata, suter2019robustly}. Given the label $Y$, the $d$-dimensional observational data $\mathbf{X}$ is consist of causal factors including the parents $\mathbf{pa_Y}$, non-descendants $\mathbf{nd_Y}$, descendants $\mathbf{dc_Y}$ of $\mathbf{Y}$. The causal information $\mathbf{pa_Y}$ enables the model a better generalization and robustness for prediction tasks. We consider the natural data generative process as an information propagation along the causal graph and try to find out $\mathbf{pa_Y}$ from $\mathbf{X}$. Based on the causal modelling, we propose to learn latent representations which maintain the most necessary causal information for the prediction task, named minimal sufficient causal information of a system.

More specifically, we define the minimal sufficient cause (MSC) $\mathbf{Z}$ as a proxy of the parents in factor space as shown in Fig. \ref{fig:intro} (b). MSCs are variables that are specially positioned in the system, blocking the path from the causes and non-descendants to $Y$. In this paper, we implement it by an information-theoretical approach, reducing the traditional two-step procedure i.e. causal disentanglement and information minimizing, to an optimization problem that can directly learn a latent causal representation with minimal sufficiency from observations. Specifically, the proposed optimization problem is a bi-level optimization problem minimizing $I(\mathbf{Z}; \mathbf{pa_Y, nd_Y})$, with maximizing mutual information $I(\mathbf{Z}; Y)$ as a constraint. Based on this, we propose an intervention effect to accurately specify the causal information $\mathbf{pa_Y}$. We name this method as \textbf{CaRI} (learning \underline{Ca}use \underline{R}epresentation by \underline{I}nformation-theoretic approach) and we further extend the method under robustness learning framework.  Moreover, we theoretically analyze the sample efficiency of CaRI by giving a generalization error bound with respect to sample size. Experiments on synthetic and real-world datasets show the effectiveness of the proposed method.

\begin{figure*}[ht!]
    \begin{center}
    \centerline{\includegraphics[width=1.6\columnwidth]{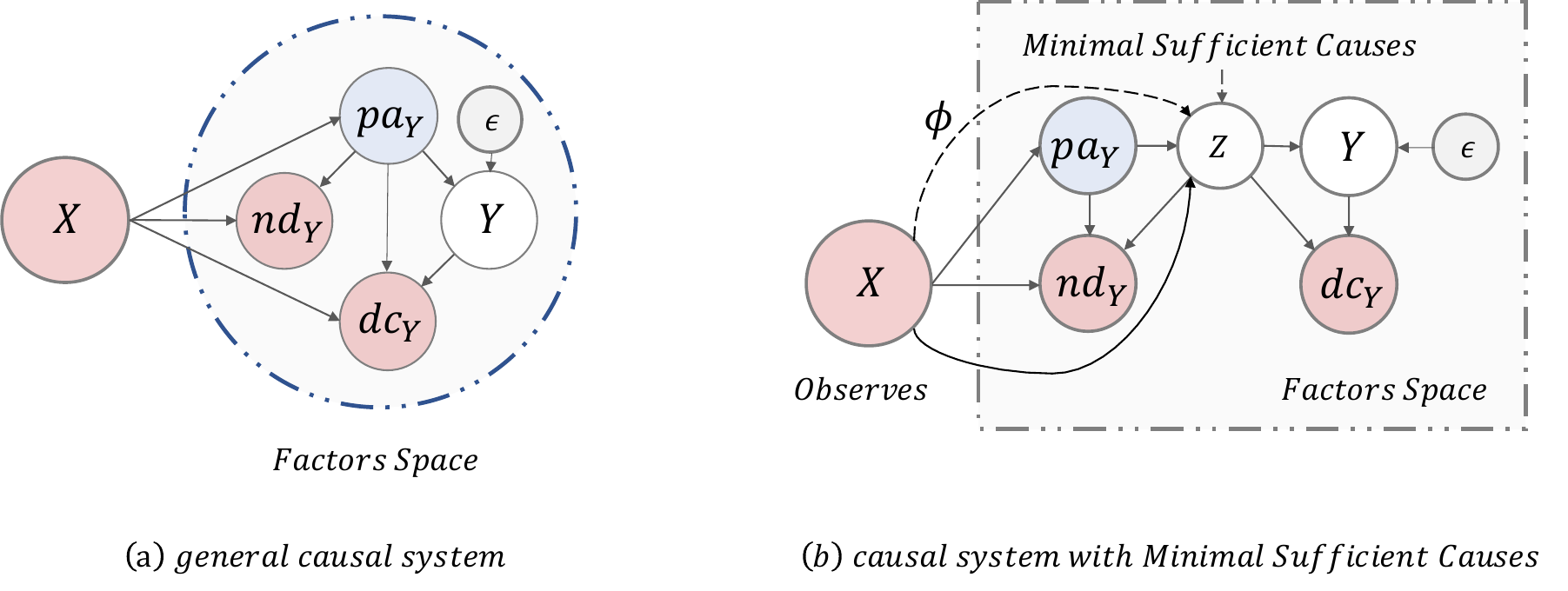}}
    \caption{The figure demonstrates a case of a causal system (a) and its extension of introducing minimal sufficient causes (b).}
    \label{fig:intro}
    \end{center}
\end{figure*}

The main contribution of this paper are summarized below:
\begin{itemize} [topsep = 3pt,leftmargin =10pt]
    \item We define minimal sufficient causes (MSC) in causal system by the formalization of an explicit causal graphical model to describe the data generative process of the real-world system and propose an information-theoretical approach to learn MSC from observational data.
    \item We theoretically analyze the sample efficiency of the learning approach by giving a generalization error bound w.r.t sample size. The theorem depicts a quantitative link between the amount of causal information contained in the learned representation and the sample complexity of the model on downstream tasks.
    \item We empirically verify that CaRI is able to generalize well distribution shift respectively and robust against adversarial attack.
\end{itemize}

\section{Related Works}
Causal Representation Learning is a set of approaches to finding generalizable representations by extracting and utilizing causal information from observational data. They usually aim at finding causal structure and causal variables behind observations. From several different perspectives, a bunch of methods have been proposed in the literature. 

\textbf{Causal Structure Learning.} To assess the connection between causally related variables in real-world systems, a bunch of traditional methods use the Markov condition and conditional independence between cause and mechanism principle (ICM) to discover the causal structure or distinguish causes from effect \cite{mooij2016distinguishing}. Several works focus on the asymmetry between cause and effect. \cite{sontakke2021causal, steudel2010causal, janzing2010causal, cover1999elements}, and similar ideas are utilized by \cite{parascandolo2018learning, steudel2010causal}. The series of works always assume that all the variable is observable. In contrast with these works, our proposed method is applicable to scenarios where the observed data is generated by hidden causal factors.

\textbf{Invariant Representation Learning Cross Multidomain.}
Some pioneering work \cite{zhou2021domain, wang2021desiderata,shen2021towards} considers the heterogeneity across multiple domains under the out-of-distribution settings \cite{gong2016domain, li2018domain, magliacane2017domain, zhang2015multi, lu2021nonlinear, rojas2018invariant, meinshausen2018causality, peters2016causal}. They learn causal representations from observational data by enforcing invariant causal mechanisms between the causal representation and the task labels across multi-domains. Similar to these works, we target obtaining invariant latent causal information but do not assume that the datasets are collected from multi-domains. 

\textbf{Causal Disentanglement Representation Learning.} Causal representation learning helps to reduce the dimension of the original high-dimensional input. Several works leverage structural causal models to describe causal relationships inside the entangled observational data \cite{yang2021causalvae, shen2020disentangled, wang2021desiderata} and learn to disentangle causal concepts from original inputs. Different from aforementioned works, the proposed method in this paper considers the causal information from the perspective of information theory \cite{DBLP:conf/icml/BelghaziBROBHC18, DBLP:conf/icml/ChengHDLGC20}. We put our attention on minimal causal information, which can be regarded as a compact representation of the whole underlying causal system. We also theoretically analyze the generalization ability from PAC learning frameworks \cite{shalev2014understanding, shamir2010learning} and explain why the causal representation can achieve better generalization ability from the perspective of sample complexity.

\section{Problem Definition}

\label{sec:main_obj}

\subsection{Notations}
Considering the causal scenario in Fig.\ref{fig:intro} (a), the observation data can be generated by the concepts in hidden space which contain multiple hidden causal variables. Denote $\mathbf{X}\in\mathcal{X}$ as $d$-dimensional observational data like context information or features in real-world systems, and ${Y}\in\mathcal{Y}$ as the labels of downstream tasks. Each pair of sample $(\mathbf{x,} y)$ is drawn i.i.d. from joint distribution $p(\mathbf{x},y)$. We use $\mathbf{pa_Y}\in\mathbb{R}^{p_1}$ to denote the variables including parent nodes of $Y$ in the causal graph, while $\bm{\epsilon}$ is the vector of independent noise with probability densities of $p_{\bm{\epsilon}}=\mathcal{N}(0, \beta I)$. Similarly, $\mathbf{dc_Y}\in\mathbb{R}^{p_2}$ and $\mathbf{nd_Y}\in\mathbb{R}^{p_3}$ denote the descendant and non-descendant nodes of $\mathbf{Y}$, respectively. In our method, we introduce minimal sufficient parents, denoted by $\mathbf{Z}\in\mathcal{Z}$ of the system.

Note that all the causal factors are assumed to be embedded in factors space, the observed data only contains $(\mathbf{X}, Y)$, where $\mathbf{X} = \mathbf{h}(\mathbf{pa_Y}, \mathbf{nd_Y},  \mathbf{dc_Y}), \mathbf{h}\in \mathcal{H}$ where $\mathbf{h}: \mathbb{R}^{p1+p2+p3}\rightarrow\mathbb{R}^{d}$ is a deterministic function.  In causal systems, the causes of prediction tasks are stable and robust, this means that when intervening on the parents, the causal effect is propagated to its child but not vice versa. 
All other correlated variables $\mathbf{nd_Y, dc_Y}$ in the causal system are regarded as spurious-correlated variables.

\subsection{Minimal Sufficient Causes (MSC)}
In our paper, we claim that not all the cause information is useful for prediction tasks. For example, considering a case of burning fire in a room, it is the presence of oxygen which explain the fire, but the match struck is definitely the necessary cause of fire. This real-world example is selected from section 9 in \cite{pearl2009causality}. From the perspective of finding the most useful causes from observational data, we introduce the minimal sufficient cause variable $\mathbf{Z}$ into the causal system. As Fig. \ref{fig:intro} (b) shows, the minimal sufficient causes $\mathbf{Z}$ are regarded as the proxy of parent variables. We define minimal sufficient causes in detail as below.
\begin{definition} \label{def:condition_independent}Assuming that the causal graph (Fig. \ref{fig:intro} (b)) with Minimal Sufficient Causes holds,
the Minimal Sufficient Cause blocks the path between $[\mathbf{pa_Y}, \mathbf{nd_Y}]$ and $Y$, and the following conditional independence condition holds:
\begin{equation}
\begin{split}
 (\mathbf{pa_Y}, \mathbf{nd_Y})\perp Y |\mathbf{Z}
\end{split}
\end{equation}
\end{definition}
\noindent Our goal is to identify the minimal sufficient information $\mathbf{Z}$ in hidden factors space. The minimum sufficient causal variable $\mathbf{Z}$ in a causal system is stable information for predicting $y$. From the perspective of sufficient causes, we define it from a probabilistic view, which is inspired by the minimal sufficient statistics \cite{lehmann2012completeness}.

\begin{definition}\label{def:sufficient}
(Sufficient Causes). Let $\mathbf{pa_Y, nd_Y, Z}, Y$ be random variables. $\mathbf{Z}$ is sufficient cause of $Y$ shown in Fig. \ref{fig:intro} (b) if and only if 
\begin{equation}
    \begin{split}
        \forall \mathbf{pa_y \in {pa_Y}}, \mathbf{nd_y \in nd_Y}, {y \in {Y}}; p(y|\mathbf{z, pa_y, nd_y}) = p(y|\mathbf{z})
    \end{split}
\end{equation}
\end{definition}
\noindent The definition of sufficient causes are the variables that are able to “produce” the causal system. From the perspective of minimal, we define a variable which can generate the whole the system with the minimum information. That is, all the variables of prediction task can be inferred if minimal sufficient causes are given. 
\begin{definition}\label{def:minimal_sufficient}
(Minimal Sufficient Causes). The sufficient cause $\mathbf{Z^*}$ is minimal if and only if for any sufficient cause $\mathbf{Z}$, there exists a deterministic function $\mathbf{f}$ such that $\mathbf{Z^* = f(Z)}$ almost everywhere w.r.t. $\mathbf{X}$.
% , where $\mathbf{Z^*}$ has the same causal effect with $\mathbf{Z}$.
\end{definition}
\noindent Definition \ref{def:minimal_sufficient} shows that the Minimal Sufficient Causes $\mathbf{Z}$ in the causal system is the variable containing minimal information from all parents.

\subsection{Learning MSC as Causal Representation from Observational Data}
\noindent This paper focuses on causal representation learning, which aims at finding a low-dimensional representation of observation benefiting for predicting $Y$. Fig. \ref{fig:intro} (a)(b) shows the causal system behind a prediction task, which uses observational data $\mathbf{X}$ to predict the target $Y$. The method is treated as a two stage process, and the first stage is to extract the representation from observational data. Let $\mathbf{Z} = \phi(\mathbf{X})$ denote representation extracted from original observation $\mathbf{X}$, where $\phi: \mathcal{X}\rightarrow\mathcal{Z}$ is the representation extraction function. The next stage is to use the representation to predict $Y$.

Now that we have formally defined Minimal Sufficiency, the basic objective is defined as learning a representation where all the information from minimal sufficient causes is included. The process is to model a flow of representation learning method and downstream prediction by satisfying Definition \ref{def:sufficient} \ref{def:minimal_sufficient}. The objective from Definition \ref{def:sufficient} is easy to be evaluated by common statistic methods, like independent testing by mutual information. However, it is very hard to get the minimal variable in Definition \ref{def:minimal_sufficient}. To evaluate the objectives in Definition \ref{def:sufficient} \ref{def:minimal_sufficient} in a unified framework, we utilize the information-theoretic ways since it can naturally combine Definition \ref{def:sufficient} and \ref{def:minimal_sufficient} by considering the information contained in MSC.

\section{Learning Minimal Sufficient Causal Representations}

In this section, we present a method to learn the minimal sufficient parent's information $\mathbf{Z}$ from observational data $\mathbf{X}$. The difficulty lies in distinguishing minimal sufficient cause $\mathbf{Z}$ from $\mathbf{X}$, when we only observe $\mathbf{X}$. We first analyze the information propagation among different causal variables under two typical causal graphs in hidden factors space, based on which we propose an objective function with mutual information constraints. Next, we extend our method by introducing do-operation, which can enhance the ability to distinguish causes if such information is not embedded in the observational data.

\subsection{Information-theoretic property of MSC in factor space}

An important fact is that in Fig.\ref{fig:intro} (b), the minimal sufficient causes in observational data $\mathbf{X}$  dominate the generative process of the causal system defined in Fig.\ref{fig:intro} (b). If there exists a mapping from $\mathbf{X}$ to $\mathbf{Z}$, it is a function that finds the minimal sufficient causes inside the causal system. We develop an algorithm to learn representations based on such hypothetical structure Fig. \ref{fig:intro} (b).   
Based on the definition of $\mathbf{Z}$, denoted by $I(\cdot, \cdot)$ the mutual information, we obtain the following Theorem (The proof is provided in supplementary material). 

\begin{theorem}\label{thm:sufficient_statistics}
Let $\mathbf{Z}\in\mathcal{Z}$, $\mathbf{Z}=\phi(\mathbf{X})$, $\mathbf{X} = \mathbf{h(pa_Y, nd_Y, dc_Y)}$ and $\mathbf{h}\in\mathcal{H}$ is an invertible function, $\mathbf{Z}$ is a minimal sufficient cause of the causal system demonstrated in Fig. \ref{fig:intro} (b) if and only if $\mathbf{Z}$ is an optimal solution of following objective
\begin{equation}
\begin{split}
    & \min_{\mathbf{Z}} I(\mathbf{Z; pa_Y, nd_Y})   \\
    & \ \ s.t.~~~\mathbf{Z} \in {\arg\max}_{\mathbf{Z'}} I(\mathbf{Z'}; Y)
\label{eq:thm1}
\end{split}
\end{equation}
\end{theorem}

\noindent Theorem \ref{thm:sufficient_statistics} shows that we can identify the MSC by solving the min-max optimization problem. In real-world applications, the information of $\mathbf{nd_Y}$ and $\mathbf{dc_Y}$ may not be revealed, and the above objective function cannot be optimized directly. To get a tractable form, in the next section, we extend our optimization objective to observational space. We extend Eq. \ref{eq:thm1} to a tractable objective by scaling the mutual information terms in Eq. \ref{eq:thm1}. The way is to link the unrevealed variables $\mathbf{nd_Y, dc_Y}$ to observation $\mathbf{X}$. The following lemma can help us scale Eq. \ref{eq:thm1}.
\begin{lemma}\label{lem:ine_i}
Suppose the features and labels are $\mathbf{X}, Y$ respectively, where $\mathbf{X}$ deterministically consists of the minimal sufficient parents, descendants and non-descendant as $\mathbf{X} = \mathbf{h(pa_Y, nd_Y, dc_Y)}$. The following inequality holds if and only if $\mathbf{h}$ is an invertible deterministic function
\begin{equation}
I(\mathbf{Z; nd_Y, pa_Y})\le I(\mathbf{Z;X})
\end{equation}
\end{lemma}
\begin{proposition}\label{thm:sufficient_statistics_mix}
Let $\mathbf{Z'}, \mathbf{Z}\in\mathcal{Z}$, $\mathbf{Z}=\phi(\mathbf{X})$, $\mathbf{X} = \mathbf{h(pa_Y, nd_Y, dc_Y)}$, $h\in\mathcal{H}$ $\mathbf{h}(\cdot)$ is  invertible function. When all the functions (lines in Fig. \ref{fig:intro}) between $\mathbf{pa_Y, nd_Y, dc_Y}$ are invertible $\mathbf{Z}$ is the minimal sufficient cause of the causal system demonstrated in Fig. \ref{fig:intro} (b) if and only if $\mathbf{Z}$ equals to the optimal solution of following objective
\begin{equation}
\begin{split}\label{eq:minmax_x}
     & \min_{\mathbf{Z}} I(\mathbf{Z; X}),  \\
     & \ \ s.t.~~~\mathbf{Z} \in {\arg\max}_{\mathbf{Z'}} I(\mathbf{Z'}; Y)
\end{split}
\end{equation}
\end{proposition}

\noindent From Theorem \ref{thm:sufficient_statistics_mix} we can substitute the terms including $\mathbf{nd_Y}$ and $\mathbf{pa_Y}$ by tractable mutual information term. 
For Eq. \ref{eq:minmax_x} in Proposition \ref{thm:sufficient_statistics_mix}, by defining $m(Y) = {\max}_{\mathbf{Z'}} I(\mathbf{Z'}; Y)$ and then reformulating $\mathbf{Z} \in {\arg\max}_{\mathbf{Z'}} I(\mathbf{Z'}; Y)$ as $I(\mathbf{Z}; Y) \geq m(Y)$, with plugging in $\mathbf{Z} = \phi(\mathbf{X})$ the optimization problem in Eq. \ref{eq:thm1} can be equivalently formulated as minimizing the following Lagrangian $\mathcal{L}(\phi, \lambda)$ such that
\begin{equation}\label{eq:delta}
\begin{split}
    % \delta(\phi)  &= 
    % % \max_\phi {I(\phi(\mathbf{X}):Y, \mathbf{nd_Y})} + H(\mathbf{pa_Y}) - \lambda{I(\phi(\mathbf{X):nd_Y, dc_Y)}}\\
    % % &\ge 
    % \max_\phi {I(\mathbf{Z};Y)}  - \lambda{I(\mathbf{Z; pa_Y, nd_Y})}
    \delta(\phi) = \mathcal{L}(\phi, \lambda) = I(\phi(\mathbf{X});\mathbf{X}) - \lambda I(\phi(\mathbf{X}); Y),
\end{split}
\end{equation}
where $\mathcal{L}(\phi, \lambda)$ is denoted as $\delta(\phi)$ for conciseness and $\lambda$ is a hyperparameter which is manually selected in practice. The object coincides with Information Bottleneck (IB) objective function \cite{shamir2010learning}. The difference is that IB is deduced from Rate Distortion Theorem in information theory, and it holds under the structure of the Markov Chain instead of a causal graph (i.e. Fig. \ref{fig:intro}). In this paper, the IB setting is generalized into causal space, by bridging minimal sufficient causes with root cause variables in the hypothetical causal graph. The detailed proof of Theorem \ref{thm:sufficient_statistics} and Proposition \ref{thm:sufficient_statistics_mix} in supplementary materials show the differences between our proposed method and IB.

\subsection{Distinguishing components by intervention effect}
The previous section illustrates a method to find $\mathbf{Z}$ on the factor and  observation level. Note that the objective function $\delta(\phi)$ (Eq. \ref{eq:delta}) given from Proposition \ref{thm:sufficient_statistics_mix} can find the minimal sufficient causes under the strong assumption. In real-world applications, if we use the information-theoretic objective, it is very hard to distinguish causes $\mathbf{pa_y}$ and spurious variable $\mathbf{dc_Y, nd_y}$ of $Y$ from the objective.

To release this problem, we introduce an intervention operation, denoted by $do(X=x)$ \cite{pearl2009causality} into our method. Intervention in
causality means the system operates under the condition
that certain variables are controlled by external forces. In hidden factor space, one of the differences between $\mathbf{Z}$ and $\mathbf{dc_Y, nd_y}$ is that if we intervene on the value of $\mathbf{Z}$, the causal effect will be delivered to its child $Y$, but the causal effect to $Y$ from the intervention conducted on child node $\mathbf{dc_Y}$ will be blocked. From such, let $\mathbf{\bar{x}}$ means the intervened value which not equals to $\mathbf{x}$, the following inequality describes intervention effect holds
\begin{equation}\label{eq:intervention}
\begin{split}
    &P(Y=y|do(\mathbf{Z = z}))-P(Y=y|do(\mathbf{Z = \bar{z}})) >\\
    &P(Y=y|do(\mathbf{dc_Y = dc_y})))-P(Y=y|do(\mathbf{dc_Y = \bar{dc_y}})) = \\
    &P(Y=y|do(\mathbf{nd_Y = nd_Y})))-P(Y=y|do(\mathbf{nd_Y = \bar{nd_Y}})) = 0
\end{split}
\end{equation}
Instead of conducting interventions on the parental variables in a real-world environment, we create a representation space $\mathcal{Z}$ where it supports simulation of the interventional manipulation on parents by intervening $\mathbf{Z}$ in the learned model. The functional interventional distributions $P(Y=y|do(\mathbf{Z = \phi(\hat{\mathbf{x}})}))$ can be identified from purely observational data $\mathbf{X}$ and $Y$ ( \cite{pearl2009causality, wang2021desiderata, puli2020causal}), 
\begin{equation}
\begin{split}
    &P(Y=y|do(\mathbf{Z = \phi(\hat{\mathbf{x}})})) \\
    &= \int_\mathbf{x}P(Y=y|\mathbf{x}, \hat{\mathbf{z}})|_{\hat{\mathbf{z}} = \phi(\hat{\mathbf{x}})}P(\mathbf{\hat{z}}|\mathbf{pa_y})P(\mathbf{pa_y}|\mathbf{x})P(\mathbf{X = x}) d\mathbf{x} \\
    &= E_\mathbf{x}[P(Y=y|\mathbf{ \hat{z}})]|_{\hat{\mathbf{z}}= \phi(\hat{\mathbf{x}})}
\end{split}
\end{equation}

\noindent Therefore in the representation space, we can directly maximize the intervention effect on the intervention space $\mathcal{Z}$ to satisfy Eq. \ref{eq:intervention}. To make the intervention effect easier to be evaluated in the mutual information process, we introduce an intervention variable $\mathbf{\bar{Z}}\in|\mathcal{Z}|$ and build an intervention network shown in Fig. 2, in which we first infer the representation $\mathbf{z}$ from observational data $\mathbf{x}$, based on which we can obtain the intervened value $\mathbf{\bar{z}} \not= \mathbf{z}$. Then we optimize the parameters in the model by maximizing the intervention effect term defined in mutual information language by 
\begin{equation}
\begin{split}
    \text{Intervention Effect} &= \int_{\mathbf{z}, y}p(\mathbf{z}, y)\log p(y|\mathbf{z}) dy d\mathbf{z}  \\
    &-\int_{y, \mathbf{z}}p(y, \mathbf{\bar{z}})\log p(y|\mathbf{\bar{z}}) dy d\mathbf{\bar{z}}  \\&= I(\mathbf{Z};{Y})-I(\mathbf{\bar{Z}};{Y})
\end{split}
\end{equation}
Integrating intervention effect and the objective function Eq. \ref{eq:delta}, the final objective is defined as below. The additional term $I(\mathbf{\bar{Z},} Y)$ is the key to evaluating the intervention effect. 

\begin{equation}\label{eq:ib}
\begin{split}
 L(\phi) = \min_\phi \underbrace{ I(\mathbf{Z};\mathbf{X}) -  (I(\mathbf{Z};{Y})}_{\text{(1)positive term}}-\underbrace{\lambda I(\mathbf{\bar{Z}};{Y}))}_{\text{(2)negative term}}
\end{split}
\end{equation}
To intuitively understand the final objective Eq. \ref{eq:ib}, we divide it into positive and negative parts. The first positive term aims at finding minimal causes, and it helps retain information from the prediction task and drops redundant information from the original input. For the negative term, it is used to distinguish causes from all correlated variables by decreasing the information overlapping between $Y$ and intervened representation $\mathbf{\bar{Z}}$.

\section{Practical Algorithms}
In this section, we provide the details of how to evaluate the mutual information term in Eq. \ref{eq:ib} and the alternative robust training process of our method.
\subsection{Implementation of $L(\phi)$}\label{sec:exp_steup}
In this paper, all objective functions are defined under mutual information formulation. We evaluate Eq.\ref{eq:ib} in two parts. The first positive part (Eq.\ref{eq:ib} (1)) is evaluated by the following parameterized objective, the variational estimation of mutual information \cite{alemi2016deep}:
\begin{equation}\label{eq:elbo}
\begin{split}
      &I(\textbf{Z}; \mathbf{X}) - \lambda I(\mathbf{Z}; Y)
    \\
    &\ge \lambda\mathbb{E}_{D}[\mathbb{E}_{\mathbf{z}\in q(\mathbf{z|x})}[\log p_{g}({y}|\mathbf{z})]-\mathcal{D_\text{KL}}(q_{\bm{\phi}} (\mathbf{z}|\mathbf{x})||p_{\bm{\theta}}(\mathbf{z}))] 
\end{split}
\end{equation}
For the negative term described in Eq. \ref{eq:ib}, the minimization process requires the upper bound of it \cite{cheng2020club}. The upper bound is formed below:
\begin{equation}\label{eq:upper_bound}
    I(\mathbf{\bar{Z}}, Y)\le \mathbb{E}_{p(y, \mathbf{ \bar{z}})}[\log (p(y |\mathbf{\bar{z}} ))]-\mathbb{E}_{y} \mathbb{E}_{\mathbf{\bar{z}}}[\log (p(y| \mathbf{\bar{z}} ))]
\end{equation}
Note that the expectation on second term in Eq. \ref{eq:upper_bound} requires marginal distribution $p(y)$ $p(\mathbf{\bar{z}})$ rather than joint distribution, therefore we independently sample $y$ in practice. 
The intervened network (Fig. 2) helps us calculate the value of $\mathbf{\bar{Z}}$ along two steps. The first step, we build a neural network to generate the transformation vector $\mathbf{T}\in\mathbb{R}^t$ from observational data $\mathbf{X}$, where ($\mathbf{t} = \mathbf{k(x)}$) and $\mathbf{k}: \mathcal{X}\rightarrow \mathbb{R}^t$ is a deterministic function modeled by neural network. The second step, the density of intervened $\mathbf{\bar{Z}}$ is calculated by $p(\mathbf{\bar{z}|z, t}) = \delta(\mathbf{\bar{z} = z + t})$, where $\delta(\cdot)$ is Dirac delta function. In experiments, if $\mathbf{t}$ is close to 0, it will decline performance of our method since original $\mathbf{z}$ is close to the intervened one $\mathbf{\bar{z}}$. To avoid this problem, we add an additional constraint $\min_\mathbf{k} |\mathbf{t^2-b}|^2$, where $b$ is a hyperparameter, in our experiments, we set $b=0.8$.
\subsection{Robust Learning under Adversarial Attack}
\noindent To enhance the robustness against potential exogenous variables or noises $\bm{\epsilon}$ and guarantee the robustness of the proposed method, we extend our method by incorporating adversarial learning. Considering the causal generative process as $\mathbf{Y} = f(\mathbf{pa_Y}, \bm{\epsilon})$, the $\bm{\epsilon}$ is regarded as a random noise perturbing the $\mathbf{pa_y}$ inside a ball with finite diameter.  We treat the inference approach as the process of adversarial attack \citep{szegedy2013intriguing, ben2009robust, biggio2018wild} and define the influence of exogenous variables as
\begin{equation}\label{eq:action_step}
\begin{split}
\mathbf{z}' = \mathbf{z} + \bm{\epsilon}, \mathbf{z}'\in \mathcal{B}(\mathbf{z}, \beta)\\
\mathbf{\bar{z}}' = \bar{\mathbf{z}}+ \bm{\epsilon}, \mathbf{\bar{z}}'\in\mathcal{B}( \mathbf{\bar{z}}, \beta)
\end{split}
\end{equation}

where $\mathcal{B}(\mathbf{z}, \beta)$ is Wasserstein ball, in which the $p$-th Wasserstein distance \citep{panaretos2019statistical} $\mathrm{W}_{p}$ \footnote{$\mathrm{W}_{p}(\mu, \nu)=\left(\inf _{\gamma \in \Gamma(\mu, \nu)} \int_{\mathcal{Z} \times \mathcal{Z}} \Delta\left(\boldsymbol{z}, \boldsymbol{z}^{\prime}\right)^{p} d \gamma\left(z, z^{\prime}\right)\right)^{1 / p}$, $\Gamma(\mu, \nu$ is the collection of all probability measures
on $\mathcal{Z} \times \mathcal{Z}$}between $z$ and $\mathbf{z}'$ is smaller than $\beta$. {$\mathbf{z'}$ and $\mathbf{\bar{z}'}$ integrate both intervention and exogenous information.}
We further define intervention robustness (IR) to measure the worst intervention results of the intervention term in Eq.\ref{eq:ib}. IR aims at finding the worst perturbation of $\mathbf{z}$ and $\mathbf{\bar{z}}$, which is formally defined below,

\begin{definition} (Intervention Robustness) \label{def:cv_term}
Let $\mathbf{\bar{Z}}'$ denote intervened variables on $\mathbf{Z}=\phi(\mathbf{X})$, $\forall \mathbf{z}'\in \mathcal{B}(\mathbf{z}, \beta), \mathbf{\bar{z}}'\in\mathcal{B}( \mathbf{\bar{z}}, \beta)$, $D$ and $D'$ denote datasets sample from $p(\mathbf{z', y})$ and $p(\mathbf{\bar{z}', y})$, the  intervention robustness is defined as
\begin{equation}
\begin{split}
    \min_{\mathbf{z}'\in \mathcal{B}(\mathbf{z}, \beta), \mathbf{\bar{z}}'\in \mathcal{B}(\mathbf{\bar{z}}, \beta)}IR_\mathcal{B} = \min_{\mathbf{z}'\in \mathcal{B}(\mathbf{z}, \beta)}I(Y;\mathbf{Z}') - \max_{\mathbf{\bar{z}}'\in \mathcal{B}(\mathbf{\bar{z}}, \beta)}I(Y;\mathbf{\bar{Z}}') 
\end{split}
\end{equation}

\end{definition}
\begin{remark}
The intervention robustness defines the worst intervention effect influenced by exogenous $\bm{\epsilon}$. For the representation $\mathbf{z}$, the term $\min_{\mathbf{z}'\in \mathcal{B}(\mathbf{z}, \beta)}I(Y;\mathbf{Z}')$ aims at find the perturbed $\mathbf{z}'$ around $\mathbf{z}$ with lowest mutual information $I(Y;\mathbf{Z}')$. For the transformed variable $\mathbf{\bar{z}}$, IV aims to find worst mutual information $\max_{\mathbf{\bar{z}}'\in \mathcal{B}(\mathbf{\bar{z}}, \beta)}I(Y;\mathbf{\bar{Z}}')$. Combining two worst mutual information together, the IR term  aims at finding the worst intervention effect perturbed by $\bm{\epsilon}$.
\end{remark}
Combining the IR term with the original objective $L(\phi)$, we get the final objective function optimized by minmax approach. Equivalently, we only need to optimize $I(\mathbf{Z'};Y)$ rather than $I(\mathbf{Z};Y)+I(\mathbf{Z'};Y)$ since if the worst case $I(\mathbf{Z}';Y)$ is satisfied, $I(\mathbf{Z};Y)$ is satisfied. The  robust optimization objective function is $L_\text{rb}({\phi})$, where
\begin{equation}\label{eq:l_rb}
\begin{split}
    &\max_{\phi} \min_{\mathbf{z}'\in \mathcal{B}(\mathbf{z}, \beta), \mathbf{\bar{z}}'\in \mathcal{B}(\mathbf{\bar{z}}, \beta)} L(\phi) + IR_\mathcal{B} \\
    \ge&\max_{\phi} \min_{\mathbf{z}'\in \mathcal{B}(\mathbf{z}, \beta), \mathbf{\bar{z}}'\in \mathcal{B}(\mathbf{\bar{z}}, \beta)} \underbrace{I(\mathbf{Z'};Y)-\lambda I(\mathbf{Z;X})}_{\text{(1) positive}} -\underbrace{I(\bar{\mathbf{Z}}';Y) }_{ \text{(2) negative}} \\
    =&L_\text{rb}(\phi)
\end{split}
\end{equation}

\noindent The inequality in above objective is due to $$I(\mathbf{Z'};Y)-I(\mathbf{\bar{Z}'};Y)\ge \min_{\mathbf{z}'\in \mathcal{B}(\mathbf{z}, \beta), \mathbf{\bar{z}}'\in \mathcal{B}(\mathbf{\bar{z}}, \beta)}IR_\mathcal{B}.$$ The robust method is learned by the minimax procedure. Literally, the minimization procedure helps to avoid the worst-case led by exogenous variable $\bm{\epsilon}$, because it maximizes the intervention robustness by adjusting the parameter of feature extractor $\phi$. The optimization  objective of the robust method can extract minimal sufficient causal representation from observation data with high robustness ability.

We train and evaluate the robust method by the adversarial attack on representation space.
We use PGD attack \cite{madry2017towards} with $\infty$-norm and $2$-norm to get intervened $\mathbf{z}'$ and $\mathbf{\bar{z}'}$. We set $p_{\bm{\theta}}(\mathbf{z})$ as $\mathcal{N}(y, 1)$ to avoid trivial representations. Then we use 
negative cross entropy to approximate mutual information. More implementation details are shown in the supplementary material.

\section{Why Causal Representation Can Enhance Generalization Ability}
In this section, we theoretically analyze the generalization property of causal representation by learning theory framework \cite{shalev2014understanding}. Learning theory contains a set of methodologies to show the upper bound of the gap between risk/error on training data and all possible data from the data distribution. These methods justify a generalization problem that whether a model learned from a small data set can be generalized to any unseen test data from data distribution. Instead of estimating the risk bound, we start from the perspective of information theory and follow the framework of information bottleneck \cite{shamir2010learning}. We provide a finite sample bound of the difference between ground truth and estimated one, which measures the generalization ability. The bound  the relationship between $I(\mathbf{Z}; Y)$ and its estimation $\hat{I}(\mathbf{Z}; Y)$. 
\subsection{The Generalization Error Bound of i.i.d. Data}
Here, we provide theoretical justification with the following theory (The proof is provided in supplementary material):

\begin{theorem}\label{thm:sample_complexity}
Let $\mathbf{Z}=\phi(\mathbf{X})$ where $\phi: \mathcal{X}\rightarrow\mathcal{Z}$ be a fixed arbitrary function, determined by a known conditional probability distribution $p(\mathbf{z|x})$. Let $m$ be sample size and $C$ is a constant. For
any confidence parameter $0<\delta<1$, it holds with a probability of at least $1 - \delta$, that

1. General case (The learned representation $Z$ contains correlated information)
    {\small{\begin{equation}
    \begin{aligned}
        &|I(Y ; \mathbf{Z})-\hat{I}(Y ; \mathbf{Z})| \\&\leq
        \frac{\sqrt{C \log (|\mathcal{Y}| / \delta)}\left(|\mathcal{Y}|\sqrt{|\mathcal{Z}|} \log (m)+\frac{1}{2} \sqrt{|\mathcal{Z}|} \log (|\mathcal{Y}|)\right)+\frac{2}{e}|\mathcal{Y}|}{\sqrt{m}}
    \end{aligned}
    \end{equation}}}
    where $m \geq \frac{C}{4} \log (|\mathcal{Y}| / \delta)|\mathcal{Z}|\mathrm{e}^{2}$
    
2. Ideal case (The learned representation $\mathbf{Z}$ contains information of causes)
    {\small{\begin{equation}
    \begin{split}
        &|I(Y ; \mathbf{Z})-\hat{I}(Y ; \mathbf{Z})|
        \\&\leq \frac{\sqrt{C \log (|\mathcal{Y}| / \delta)}\left(|\mathcal{Y}|\sqrt{\beta} \log (m)+\frac{1}{2} \sqrt{|\mathcal{Z}|} \log (|\mathcal{Y}|)\right)+\frac{2}{e}|\mathcal{Y}|}{\sqrt{m}}\nonumber
    \end{split}
    \end{equation}}}
    where $m \geq C \log (|\mathcal{Y}| / \delta)\beta\mathrm{e}^{2}$
\end{theorem}
\begin{remark}
The theorem provides a generalization bound under finite sample settings. It shows that when representation $\mathbf{Z}$ fully contains parent information $\mathbf{pa_Y}$, we achieve a sample complexity bound as $m \geq C \log (|\mathcal{Y}| / \delta)\beta\mathrm{e}^{2}$, where $\beta$ is the variance of $\bm{\epsilon}$. The minimum number of samples needed reduces from $|\mathcal{Z}|$ to $\beta$, which is a tighter bound since in most of cases we assume $|\mathcal{Z}|\gg\beta$. 
%Sample complexity reduced and generalization error bound tighten in the ideal case when $\mathbf{z=pa_Y}$. 
This shows that $\mathbf{z=pa_Y}$ gives the reduced sample complexity and tightened generalization bound. The theorem also serves as a general solution to causality prediction problems, supporting the claim that a better prediction is achieved with causal variables, compared to that with correlated variables. 

\subsection{The Generalization Error Bound when Distribution Shift Happens}
We also show additional generalization results. For the scenario of distribution sift, we define the mutual information on source domain as $I_\mathcal{S}(\mathbf{Z}, Y)$ and mutual information on target domain as $I_\mathcal{T}(\mathbf{Z}, Y)$. Denote joint distribution in source and target domain as  $\mathcal{S}(\mathbf{z}, y) = p_\mathcal{S}(\mathbf{z}, y)$ and $\mathcal{T}(\mathbf{z}, y) = p_\mathcal{T}(\mathbf{z}, y)$, separately.

\begin{assumption}\label{asp:distribution_shift}
The causal mechanism $p\mathbf{(y|z)}$ and causal representation $p(\mathbf{z})$ are stable under distribution shift such that $p_\mathcal{S}\mathbf{(y|z=\phi(x))}=p_\mathcal{T}\mathbf{(y|\mathbf{z}=\phi(\mathbf{x}))}$ and $p_\mathcal{S}(\mathbf{z}=\phi(\mathbf{x})) = p_\mathcal{T}(\mathbf{z}=\phi(\mathbf{x}))$, if $\mathbf{Z}$ is sufficient cause of $Y$. 
\end{assumption}
When the invariant assumption holds, we can connect $I_\mathcal{T}(Y ; \mathbf{Z})$ and $\hat{I}_\mathcal{T}(Y ; \mathbf{Z})$ by following theorem.
\begin{theorem}
Let $\mathbf{Z}=\phi(\mathbf{X})$ where $\phi: \mathcal{X}\rightarrow\mathcal{Z}$ be a fixed arbitrary function, determined by a known conditional probability distribution $p(\mathbf{z|x})$. Let $m$ be sample size and $C$ is a constant. In domain adaptation scenario, defining $D_{KL}(\mathcal{S}||\mathcal{T})>0$ as the Kullback-Leibler divergence between source domain and target domain. For
any confidence parameter $0<\delta<1$, it holds with a probability of at least $1 - \delta$, that

1. General case (The learned representation $Z$ contains correlated information)
{\small{\begin{equation}
\begin{split}
    &|I_\mathcal{T}(Y ; \mathbf{Z})-\hat{I}_\mathcal{T}(Y ; \mathbf{Z})| \\&\leq
    \frac{\sqrt{C \log (|\mathcal{Y}| / \delta)}\left(|\mathcal{Y}|\sqrt{|\mathcal{Z}|} \log (m)+ D_{KL}(\mathcal{T}||\mathcal{S}) + D I_\mathcal{S})\right)+\frac{2}{e}|\mathcal{Y}|}{\sqrt{m}}
\end{split}
\end{equation}}}
    
2. Ideal case (The learned representation $\mathbf{Z}$ contains information of sufficient causes of $Y$, Assumption \ref{asp:distribution_shift} holds)
{\small{\begin{equation}
    |I_\mathcal{T}(Y ; \mathbf{Z})-\hat{I}_\mathcal{T}(Y ; \mathbf{Z})| \leq
    \frac{\sqrt{C \log (|\mathcal{Y}| / \delta)}\left(|\mathcal{Y}|\sqrt{|\beta|} \log (m)+  D I_\mathcal{S})\right)+\frac{2}{e}|\mathcal{Y}|}{\sqrt{m}}
\end{equation}}}
where $D = \frac{2}{\min _{\mathbf{z}} p(\mathbf{z})}$ and $I_\mathcal{S} = \mathbb{E}_{\mathcal{S}(\mathbf{z}, y)} \frac{\hat{p}(\mathbf{z}, y)}{\hat{p}(\mathbf{z}) \hat{p}(y)}$
\end{theorem}
\begin{remark}

The theorem shows that in a domain adaptation scenario, causal representation can help to achieve better generalization ability. We bound the risk of mutual information evaluation on the target domain by the bound on the source domain. It is because, in the training process, the information from the target domain is not observable. From the bounds of $|I_\mathcal{T}(Y ; \mathbf{Z})-\hat{I}_\mathcal{T}(Y ; \mathbf{Z})|$ shown in the general case and ideal case, we can see that the generalization error bound of the ideal case is smaller than that of the general case, with a margin quantified by a positive term $D_{KL}(\mathcal{S}||\mathcal{T})>0$. These theoretical results support that the causal representation can achieve better generalization ability under distribution shift.
\end{remark}

\end{remark}

\section{Experiments}
In this section, we conduct extensive experiments to verify the effectiveness of our framework. In the following, we begin with the experiment setup, and then report and analyze the results.

\begin{table*}\label{tab:generalization}
\caption{Overall Results on Yahoo!R3-OOD, Yahoo!R3-i.i.d. and PCIC}
    \center
\small
\renewcommand\arraystretch{1.1}
\setlength{\tabcolsep}{3.3pt}
\begin{threeparttable}  
\scalebox{1}{
    \begin{tabular}{c|c|cc|cc|cc|cc}
    \hline\hline
    Dataset& Method&\multicolumn{4}{c|}{p=$\infty$}&\multicolumn{4}{c}{p=2} \\ \hline
        & Metrics & AUC & ACC & advAUC & advACC & AUC & ACC & advAUC & advACC \\ \hline
        \multirow{7}{*}{Yahoo!R3-OOD}&base(robust) & 0.5 & 0.4508 & 0.5 & 0.4508 & 0.5 & 0.4545 & 0.5 & 0.4537 \\ 
        &base(standard) & 0.6198 & 0.6097 & 0.5212 & 0.5189 & 0.621 & 0.6099 & 0.5139 & 0.5188 \\
        &IB(standard) & 0.6181 & 0.6063 & 0.5333 & 0.5149 & 0.6184 & 0.6069 & 0.5431 & 0.5255 \\
        &r-CVAE(robust) & 0.6186 & 0.6235 & 0.5886 & 0.5912 & 0.6191 & 0.6241 & 0.5882 & 0.5907 \\
        &r-CVAE(standard) & 0.6253 & 0.6249 & 0.5855 & 0.5863 & 0.6233 & 0.6243 & 0.5865 & 0.5872 \\
        &CaRI(robust) & 0.6238 & \textbf{0.6284} & \textbf{0.5993} & \textbf{0.5999} & 0.6242 & \textbf{0.6307} & \textbf{0.6008} & \textbf{0.601} \\ 
        &CaRI(standard) & \textbf{0.629} & 0.6257 & 0.5966 & 0.5965 & \textbf{0.6276} & 0.6255 & 0.5917 & 0.5917 \\  \hline
        \multirow{7}{*}{Yahoo!R3-i.i.d.}&base(robust) & 0.5 & 0.6001 & 0.5 & 0.5997 & 0.5 & 0.6 & 0.5 & 0.6 \\ 
        &base(standard) & 0.7334 & 0.7483 & 0.6267 & 0.6251 & 0.7346 & 0.752 & 0.6260 & 0.6103 \\ 
        &IB(standard) & 0.7291 & 0.7513 & 0.6361 & 0.6721 & 0.7348 & 0.7521 & 0.6418 & 0.6775 \\ 
        &r-CVAE(robust) & 0.7341 & 0.7093 & 0.7180 & 0.7080 & 0.7376 & 0.7151 & 0.7194 & 0.7082 \\ 
        &r-CVAE(standard) & 0.7488 & 0.7515 & 0.7191 & 0.7072 & 0.7487 & 0.7529 & 0.7202 & 0.7099 \\
        &CaRI(robust) & {0.7378} & 0.7168 & \textbf{0.721} & \textbf{0.7107} & {0.7374} & {0.7158} & \textbf{0.7247} & \textbf{0.7159} \\ 
        &CaRI(standard) & \textbf{0.7497} & \textbf{0.7503} & 0.7191 & 0.7099 & \textbf{0.7493} & \textbf{0.7495} & 0.7188 & 0.7072 \\  
        \hline
        \multirow{7}{*}{PCIC}&base(robust) & 0.5534 & 0.5875 & 0.5388 & 0.6257 & 0.5605 & 0.6498 & 0.5264 & 0.6287 \\
        &base(standard) & 0.6177 & 0.6517 & 0.5231 & 0.589 & 0.6269 & 0.6615 & 0.519 & 0.5581 \\ 
        &IB(standard) & 0.6242 & 0.6532 & 0.5741 & 0.6199 & 0.6216 & 0.6537 & 0.5768 & 0.6233 \\ 
        &r-CVAE(robust) & 0.6363 & 0.6733 & 0.6088 & 0.6596 & 0.63 & 0.674 &  0.6187 & 0.6493 \\ 
        &r-CVAE(standard) & 0.6358 & 0.6779 & 0.6138 & 0.6601 & 0.6328 & 0.6725 & 0.5893 & 0.6429  \\ 
        &CaRI(robust) & {0.639} & 0.6761 & \textbf{0.6225} & { 0.6638} & {0.6363} & {0.6709} & \textbf{0.6332} & {0.6576} \\ 
        &CaRI(standard) & \textbf{0.6447} & \textbf{0.6817} & 0.6148 & \textbf{0.664} & \textbf{0.6416} & \textbf{0.6803} & 0.619 & \textbf{0.6625} \\ 
        \hline\hline
    \end{tabular}}   
\end{threeparttable}    
\label{tab:overall_yahoo}
\end{table*}

\subsection{Datasets}
{Our experiments are based on one synthetic and four real-world benchmarks. With the synthetic dataset, we evaluate our method in a controlled manner under the selected dataset. We follow the causal graph defined in Fig.\ref{fig:intro} (a) to build our synthetic simulator, on which we compare the representation learnt by our method with the ground truth under different $\beta$ degrees. 

\subsection{Synthetic Datasets}
The synthetic data is generated following the general causal graph Fig.\ref{fig:intro}. We build the simulator  using nonlinear functions refering to \cite{DBLP:conf/kdd/ZouKCC019, yang2021top}. We simulate 500 data for each settings. Let $\kappa_1(\cdot)$ and $\kappa_2(\cdot)$ as piecewise functions, and $\kappa_1(x) = x - 0.5$ if $x>0$, otherwise $\kappa_1(x) = 0$, $\kappa_2(x) = x$ if $x>0$, otherwise $\kappa_2(x) = 0$ and $\kappa_3(x) = x + 0.5$ if $x<0$, otherwise $\kappa_3(x) = 0$. . For the fair evaluation, we set the same dimension for $\mathbf{pa_Y, nd_Y, dc_Y}$ that $d_1=d_2=d_3=5$. The nonlinear systems are:
\begin{equation}
\begin{split}
    &\mathbf{pa_Y} \sim U(-1, 1), \\
    &\bm{\epsilon_1} = \bm{\epsilon_2} = \bm{\epsilon_3}\sim \mathcal{N}(0.3, \beta I)\\
    &\mathbf{nd}_{1} =  \bm{a}^T\kappa_1(\kappa_2([\mathbf{pa_Y}, \bm{\epsilon_2} ]))+{q},\\
    &\mathbf{nd}_{2} =  \bm{a}^T\kappa_3(\kappa_2([-\mathbf{pa_Y},-\bm{\epsilon_2}]))+{q}, \\
    &\mathbf{nd_Y} = \sigma(\mathbf{nd}_{1} + \mathbf{nd}_{1}\cdot \mathbf{nd}_{2})\\
    &\mathbf{y}_{1} =  \bm{a}^T\kappa_1(\kappa_2([\mathbf{pa_Y}, \bm{\epsilon_1} ]))+{q},
    \\&\mathbf{y}_{2} =  \bm{a}^T\kappa_3(\kappa_2([-\mathbf{pa_Y},-\bm{\epsilon_1}]))+{q}, \\
    &\mathbf{nd_Y} = \mathbb{I}(\sigma(\mathbf{y}_{1} + \mathbf{y}_{1}\cdot \mathbf{y}_{2}))\\
    &\mathbf{dc}_{1} =  \bm{a}^T\kappa_1(\kappa_2([y, \bm{\epsilon_3} ]))+{q},\\
    &\mathbf{dc}_{2} =  \bm{a}^T\kappa_3(\kappa_2([-y,-\bm{\epsilon_3}]))+{q}, \mathbf{dc_Y} = \sigma(\mathbf{dc}_{1} + \mathbf{dc}_{1}\cdot \mathbf{dc}_{2})\\
    &\mathbf{X} = [\mathbf{pa_Y, nd_Y, dc_Y}]\nonumber
\end{split}
\end{equation}
\begin{table*}[h]
    \center
\caption{Overall Results on Coat dataset.}
\renewcommand\arraystretch{1.1}
\setlength{\tabcolsep}{3.3pt}
\begin{threeparttable}  
\scalebox{1}{
    \begin{tabular}{c|c|cc|cc|cc|cc}
    \hline\hline
    Dataset& Method&\multicolumn{4}{c|}{p=$\infty$}&\multicolumn{4}{c}{p=2} \\ \hline
        & Metrics & AUC & ACC & advAUC & advACC & AUC & ACC & advAUC & advACC \\ \hline
        \multirow{7}{*}{Coat-OOD}&base(robust) & 0.5586 & 0.5569 & 0.5479 & 0.5451 & 0.5593 & 0.556 & 0.5441 & 0.5412 \\
        &base(standard) & 0.5659 & 0.5724 & 0.3874 & 0.4024 & 0.5642 & 0.5687 & 0.3128 & 0.3317 \\ 
        &IB(standard) & 0.5659 & 0.5681 & 0.4701 & 0.4796 & 0.5659 & 0.5713 & 0.5442 & 0.5495 \\ 
        &r-CVAE(robust) & 0.5629 & 0.5586 & 0.559 & 0.5544 & 0.5634 & 0.5591 & 0.5572 & 0.5522 \\ 
        &r-CVAE(standard) & 0.5656 & 0.5643 & 0.5527 & 0.5478 & 0.5671 & 0.5649 & 0.5586 & 0.554 \\ 
        &CaRI(robust) & \textbf{0.5707} & {0.5681} & \textbf{0.5653} & \textbf{0.5659} & {0.5705} & {0.5675} & \textbf{0.5674} & \textbf{0.565} \\ 
        &CaRI(standard) & 0.5705 & \textbf{0.5718} & 0.5643 & 0.5659 & \textbf{0.5725} & \textbf{0.5732} & 0.5608 & 0.5601 \\ 
        \hline
        \multirow{7}{*}{Coat-i.i.d.}&base(robust) & 0.7156 & 0.7232 & 0.7034 & 0.7107 & 0.7195 & 0.7261 & 0.7001 & 0.7057 \\ 
        &base(standard) & 0.7191 & 0.7217 & 0.4911 & 0.487 & 0.7235 & 0.7255 & 0.3642 & 0.3515 \\ 
        &IB(standard) & 0.7162 & 0.72 & 0.6023 & 0.6017 & 0.7182 & 0.7222 & 0.694 & 0.696 \\ 
        &r-CVAE(robust) & 0.7147 & 0.7222 & 0.7105 & 0.7181 & 0.7087 & 0.7169 & 0.7058 & 0.7141 \\ 
        &r-CVAE(standard) & 0.7106 & 0.7184 & 0.7029 & 0.7106 & 0.7129 & 0.7206 & 0.7023 & 0.7059 \\ 
        &CaRI(robust) & {0.7276} & 0.7339 & \textbf{0.7208} & \textbf{0.727} & \textbf{0.7265} & \textbf{0.7331} & \textbf{0.7196} & \textbf{0.7261} \\ 
        &CaRI(standard) & \textbf{0.7283} & \textbf{0.7355} & 0.7125 & 0.7196 & 0.7248 & 0.7305 & 0.7069 & 0.7125 \\ 
        
        \hline\hline
    \end{tabular}
}   
\end{threeparttable}    
\label{tab:coat_res}
\end{table*}
where $q=0.3$, $\mathbb{I}(x)$ is an indicator function, which is 1 if $x>0$, and 0 otherwise. From synthetic data, we analyze whether CaRI has the ability to identify the $\mathbf{pa_Y}$ from mixed observational $\mathbf{X}$.

\subsection{Real-word benchmarks}

We also evaluate our method on real-world benchmarks for the recommendation system.

\textbf{Yahoo! R3}\footnote{https://webscope.sandbox.yahoo.com/catalog.php?datatype=r} is an online music recommendation dataset, which contains the user survey data and ratings for randomly selected songs. The dataset contains two parts: the uniform (OOD) set and the nonuniform (i.i.d.) set. The non-uniform (OOD) set contains samples of users deliberately selected and rates the songs by preference, which can be considered as a stochastic logging policy. For the uniform (i.i.d.) set, users were asked to rate 10 songs randomly selected by the system. The dataset contains 14,877 users and 1,000 items. The density degree is 0.812\%, which means that the dataset only records 0.812\% of rating pairs.

\textbf{PCIC} The dataset  is
collected from a survey by questionnaires about the rate and reason why the audience like or dislike the movie. Movie features are collected from movie review pages. The training data is a biased dataset consisting of 1000 users asked to rate the movies they care from 1720 movies. The validation and test set is the user preference on uniformly exposed movies. The density degree is set to be 0.241\%. 

For evaluation, Yahoo! R3 and Coat dataset both have two validation (include test) datasets. The i.i.d. set is 1/3 of data from a nonuniform logging policy, and the OOD set consists of the data generated under a uniform policy. For the PCIC dataset, we train our method on non-uniform datasets and perform evaluations on uniform datasets. 

\textbf{CelebA-anno} The dataset contains more than 200K celebrity images, each with 40 attribute annotations. Following the previous work \cite{kocaoglu2017causalgan}, we select 9 attribute annotations, which include Young, Male, Eyeglasses, Bald, Mustache, Smiling, Wearing Lipstick, and Mouth Open. Our task is to predict Smiling. $\mathbf{pa_Y}$ including \{Young, Male\}, $\mathbf{nd_Y}$ including \{Eyeglasses, Bald, Mustache, Wearing Lipstick\} and $\mathbf{cd_Y}$ including \{Mouth Open\}. From this dataset, we evaluate the ability to distinguish $\mathbf{pa_Y}$ from $\mathbf{X}$ (Results on CelebA-anno are provided in supplementary materials).

\textbf{Coat Shopping Dataset}\footnote{https://www.cs.cornell.edu/~schnabts/mnar/} is a commonly used dataset collected from web-shop ratings on clothing. The self-selected ratings are the i.i.d. set and the uniformly selected ratings are the OOD set. In the training dataset, users were asked to rate 24 coats selected by themselves from 300 item sets. The test dataset collects the user rates on 16 random items from 300 item sets. Just as Yahoo! R3, the training dataset is a non-uniform dataset and the test dataset is a uniform dataset. The dataset provides side information on both users and item sets. The feature dimension of the user/item pair is 14/33.

\textbf{Compared Method}.  For all the compared methods, we use the same model architecture, with different training strategies. The model  consists of representation learning module $\mathbf{z}=\phi(\mathbf{x})$  and the downstream prediction module  $\mathbf{\hat{y}} = g(\mathbf{z})$, with each module implemented by neural networks. \textbf{Base} model has no additional constraints on representation, and the optimization is to minimize the cross-entropy between ${y}$ and learned ${\hat{y}}$. We  involve a recently proposed variational estimation with information bottleneck (\textbf{IB}) \cite{alemi2016deep}, extend the condition VAE (CVAE \cite{sohn2015learning}) by robust training process as \textbf{r-CVAE}, whose objective function is similar with CaRI but without a negative term (Eq.\ref{eq:ib} (2)). We conduct ablation studies by comparing our proposed method \textbf{CaRI} with the r-CVAE to evaluate the effectiveness of negative term.  We evaluate our method on two main aspects: (i) \textbf{Generalization} of the model under distribution shifts and (ii) \textbf{Robustness} under adversarial attack on representation space. 
For (i), we evaluate our method on OOD and i.i.d. setting on Yahoo! R3 and Coat. For (ii), the standard mode of adversarial attack  ($\beta=0$) means that we do not perturb original $\mathbf{z}$. In robust mode, we set $\beta=\{0.1, 0.2, 0.1, 0.3, 0.3\}$ for PCIC, Yahoo! R3, Coat, Synthetic and CelebA-anno respectively.

\textbf{Metrics}. We use the common evaluation metrics AUC/ACC \citep{rendle2012bpr, gunawardana2009survey} on CTR prediction and their variants called adv-ACC/ adv-AUC \citep{madry2017towards} on adversarially perturbed evaluation dataset. {Moreover, we consider Distance Correlation metrics \cite{jones1995fitness} to evaluate the similarity between learned representation and parental information $\mathbf{pa_Y}$. }
\vspace{-3mm}
\subsection{Implementation}
\textbf{Architecture and Setups}: The model consists of two parts, the representation learning part and the downstream prediction part. 
For the representation learning part, we first use encode function $\phi(\cdot)$ to get representation $\mathbf{z}$ and get the intervened $\mathbf{\hat{z}}$. Then we perturb the learned $\mathbf{z}$ and $\mathbf{\bar{z}}$ by PGD attack \cite{madry2017towards} procedure to find the worst case corresponding to the worst downstream loss. We use PGD attack with $\infty$-norm ($p=\infty$) and $2$-norm ($p=2$) in our implementation. Finally we put $\mathbf{z}'$ and $\mathbf{\hat{z}}'$ into the downstream prediction model $g(\cdot)$ to calculate $y$. The likelihood in Eq.\ref{eq:elbo} is estimated by cross-entropy loss.
Note that the perturbation approach would block the gradient propagation between the representation learning process and downstream prediction in some implementation ways. Thus we use the conditional Gaussian prior $p_\theta(\mathbf{z}) = \mathcal{N}(y\mathbf{1}, \mathbf{I})$ rather than standard Gaussian distribution $p_\theta(\mathbf{z}) = \mathcal{N}(\mathbf{0}, \mathbf{I})$ to calculate KL term. If gradient propagation is blocked, by using conditional prior, the learning process of representation $\mathbf{z}$ and exogenous $\bm{\epsilon}$ embedded in $\mathbf{z}'$ will not be influenced. The form of conditional Gaussian prior is more general $p_\theta(\mathbf{z}) = \mathcal{N}(\zeta(y), \mathbf{I})$, where $\zeta(\cdot)$ could be any non-trivial function like linear function even neural network.

\begin{figure*}[htp]
\begin{center}
\centerline{\includegraphics[width=1.8\columnwidth]{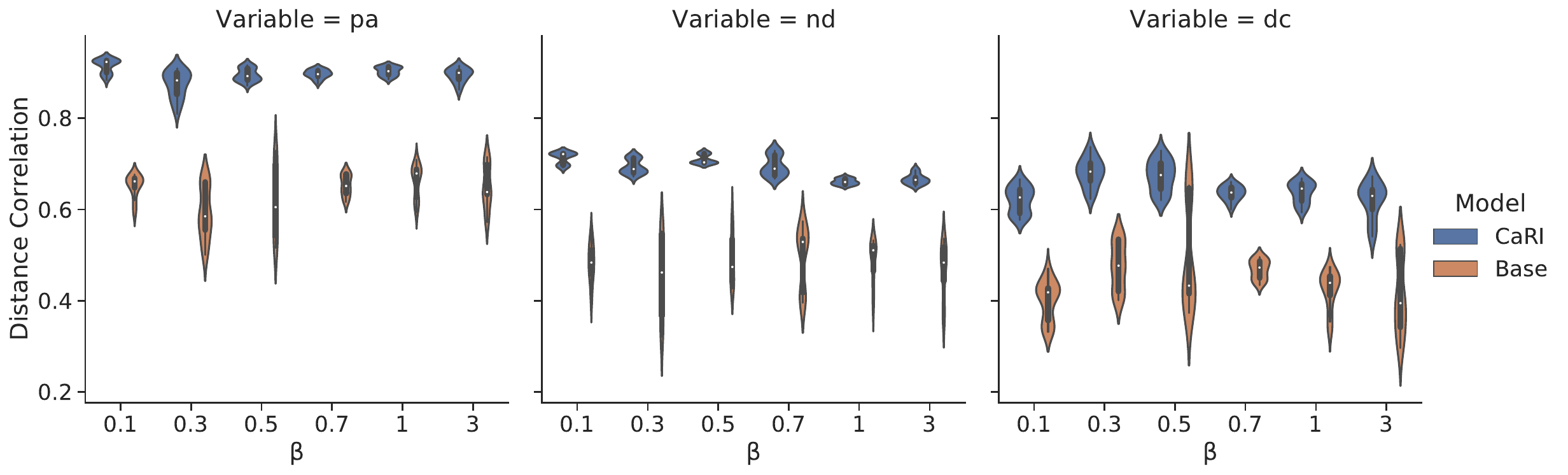}}
\caption{Representation learning results on synthetic dataset over different range of $\beta$, where $p=2$ under robust training.}
\label{fig:identify_result}
\end{center}
\vspace{-3mm}
\end{figure*}
\textbf{Hyperparameter}
The hyper-parameters are determined by grid search.
Specifically, the learning rate and batch size are tuned in the ranges of {{$[10^{-1}, 10^{-2}, 10^{-3}, 10^{-4}]$} and $[64,128,256,512,1024]$}, respectively.
The weighting parameter $\lambda$ is tuned in $[1, 100]$. Perturbation degrees are set to be $\beta=\{0.1, 0.2, 0.1, 0.3\}$ for Coat, Yahoo!R3 and PCIC separately. The representation dimension is empirically set as 64. All the experiments are conducted based on a server with a 16-core CPU, 128g memories and an RTX 5000 GPU. 
\subsection{Overall Effectiveness}
Table \ref{tab:overall_yahoo} shows the overall results on Yahoo! R3 and PCIC. From Yahoo! R3 dataset, which contains both i.i.d. and OOD validation and test sets, we find that our method enjoys better generalization. In Yahoo! R3 OOD, our method increases the performance by 1.9\%,  and 8.1\%, in terms of ACC and adv-ACC, compared with the base method. The performance of r-CVAE is close to CaRI, since it is a modified version of our method, which only includes the positive term in Eq.\ref{eq:l_rb} but removes the negative term. The difference between the performances of CaRI and r-CVAE shows the effectiveness of the negative term in the objective function of CaRI. In PCIC dataset, standard and robust modes of CaRI achieve the best AUC at 64.47\%, and 63.9\% respectively, which validates the effectiveness of our idea. In the robust training mode, our method achieves the best performance in adversarial metrics. In the PCIC dataset, our method reaches 62.25\%, which increases by 8.37\% against the base method on adv-AUC. Robust training of CaRI is also better than the standard training, winning with a margin of around 1.42\%. we present the additional test results and analysis in this section. Table \ref{tab:coat_res} shows overall experimental results on Coat, The table contains both i.i.d. and OOD settings. Based on this we find that in most cases, our method  achieves better performance in terms of AUC and ACC, compared to base methods.  The overall results show that the robust learning process with exogenous variables involved enhances the adversarial performance on perturbed samples. On the other hand, in standard training mode, CaRI achieves better adversarial performance than baselines including base method and IB. We find that standard training of CaRI on PCIC has an AUC of 64.47\%, which is better than the performance under robust training (63.9\%). But contrary conclusions are drawn on adversarial performance. The result supports that the causal representation we learned is more robust. The performance of the base method in robust training mode is worst in most of cases, indicating that the robust training process will largely influence the learning of the model and ruin the prediction model. Although the robust training deteriorates the performance of on normal dataset, it will help to identify the causal representation, which benefits downstream prediction under adversarial attack. The robust learning process with exogenous variables involved enhances the adversarial performance on perturbed samples. On the other hand, in standard training mode, CaRI achieves better adversarial performance than all baselines. The result supports that the causal representation our method learned is more robust. 

\vspace{-2mm}
\subsection{Representation Analysis}
{In this section, we study whether our method CaRI helps to identify the parental information from observational data. Fig.\ref{fig:identify_result} demonstrates the ability of the model to learn causal representations under different $\beta$ degrees on a synthetic dataset. The figure shows the distance correlation between the learned representation $\mathbf{z}$ and different parts of observational data, namely ($\mathbf{pa_Y, nd_Y, dc_Y}$). From Fig.\ref{fig:identify_result} (left), we find that our method learns a representation that is with the highest similarity, in comparison with the base method under different values of $\beta$. It is evidence that our method successfully identifies the parental information from mixed observational data. The information from $\mathbf{nd_Y}$ and $\mathbf{dc_Y}$ are not considered as important as parental information from CaRI, and the  distance correlation metric corresponding to this part is slightly lower. We also find that the metric under CaRI gets lower variance, which shows the stable performance of CaRI. On the contrary, the distance correlation metric of the base method is with high variance, which indicates the possible incapability of the base method on extracting the parental information from observations.}

\vspace{-2mm}
\section{Conclusions}
In this paper, we deal with the problem of learning causal representations from observational data, which comes with satisfactory generalization ability. Assuming that the underlying latent factors follow some causal generative models, we argue that learning a minimum sufficient cause of the system is the optimal solution. By analyzing the information theoretical property of our hypothetical graphical model, we propose a causality-inspired representation learning method by optimizing a function with regularized mutual information constraints. It achieves effective learning with guaranteed sample complexity reduction under certain assumptions. Extensive experiments on real-world dataset show the effectiveness of our algorithm, verifying our claim of robustness of the representation with respect to downstream tasks.

\vspace{-2mm}
\section{Acknowledgements}
The work was supported by National Key R\&D Program of China (2022YFB4501500, 2022YFB4501504). We thank Jianhong Wang for the fruitful discussion on some of the results and for proofreading this paper and thank Xu Chen  for proofreading. We thank the anonymous reviewers for their feedback. 
\newpage

\bibliography{acmart}
\bibliographystyle{ACM-Reference-Format}

\onecolumn
% \appendix
\onecolumn
\subsection{Proof of Theorem \ref{thm:sufficient_statistics}}
The proof directly follows the following two lemmas. We denote  the set of probabilistic functions of $\mathbf{X}$ into an arbitrary target space as $\mathcal{F}(\mathbf{X})$, and as $\mathcal{S}({Y})$ the set of sufficient statistics for ${Y}$. 
Since $\mathbf{h}(\cdot)$ is a function of combination , we have $I(Y; \mathbf{pa_Y, nd_Y, dc_Y}) = I(Y; \mathbf{X})$
\begin{lemma}
Let $\mathbf{Z}$ be a probabilistic function of $\mathbf{X}$ and $\mathbf{X} = \mathbf{h(pa_Y, nd_Y, dc_Y)}$, where $\mathbf{h}$ is function of combination. Then $\mathbf{Z}$ is a sufficient causes for ${Y}$ if and only if $$I(Y; \mathbf{Z}) = \max_{\mathbf{Z'}\in \mathcal{F}(\mathbf{X})} I(Y, \mathbf{\mathbf{Z'}})$$
\end{lemma}
\begin{proof}
Lemma is an extension of  Lemma 12 in \cite{shamir2010learning}. The differences lie on that $\mathbf{X}$ is consist of $\mathbf{pa_Y, nd_Y, dc_Y}$ and we focus on the sufficient causes defined in Definition \ref{def:sufficient}. Firstly for the sufficient condition, for every $\mathbf{Z'}$ which is a probabilistic function of $\mathbf{X}$, we have Markov Chain  $Y - \mathbf{X} - \mathbf{Z}'$, so from data processing inequality \cite{cover1999elements} we have $I(Y; \mathbf{X}) \ge I(Y; \mathbf{Z'})$. Therefore we have 
$I(Y; \mathbf{Z}) = \max_{\mathbf{Z'}\in \mathcal{F}(\mathbf{X})} I(Y, \mathbf{\mathbf{Z'}})$. We also have Markov Chain  $Y - \mathbf{Z} - \mathbf{X} $, for the data processing inequality we have $I(Y; \mathbf{X}) \le I(Y; \mathbf{Z})$. Thus $I(Y; \mathbf{Z}) = \max_{\mathbf{Z'}\in \mathcal{F}(\mathbf{X})} I(Y, \mathbf{\mathbf{Z'}})$ is held. 

Then for necessary condition, assume that we have a  Markov Chain  $Y- \mathbf{X} - \mathbf{Z}$.  According to data processing inequality, the $I(Y; \mathbf{Z}) = I(Y; \mathbf{X})$ holds if and only if $I(Y; \mathbf{X}| \mathbf{Z}) = 0$. Since $\mathbf{X}$ is consist of $\mathbf{pa_Y, nd_Y, dc_Y}$, the $I(Y; \mathbf{pa_Y, nd_Y}| \mathbf{Z}) = 0$. in other word, $Y$ and $\mathbf{pa_Y, nd_Y}$ are conditionally independent by $\mathbf{Z}$, hence $\mathbf{Z}$ is a sufficient causes satisfied Definition \ref{def:sufficient}.

\end{proof}

\begin{lemma} 
Let $\mathbf{Z}$ be sufficient statistics of  ${Y}$ and $\mathbf{X} = \mathbf{h(pa_Y, nd_Y, dc_Y)}$, where $\mathbf{h}$ is function of combination. Then $\mathbf{Z}$ is minimal sufficient causes for $\mathbf{Y}$ if and only if 
\begin{equation}\label{eq:eq_in_lemma3}
    I(\mathbf{nd_Y, pa_Y; Z}) = \min_{\mathbf{Z'}\in \mathcal{S}(Y)} I(\mathbf{nd_Y, pa_Y; Z'})
\end{equation}

\end{lemma}
\begin{proof}
Firstly, for the sufficient condition, let $\mathbf{Z}$ be a minimal sufficient causes, and  $\mathbf{Z}'$ be some sufficient causes. Because there is a function $\mathbf{Z} = f(\mathbf{Z'})$ from Definition \ref{def:minimal_sufficient}, it has Markov Chain $(\mathbf{nd_Y, pa_Y}) - {Y} - \mathbf{Z'} - \mathbf{Z}$, and we get $I(\mathbf{nd_Y, pa_Y; Z}) \le I(\mathbf{nd_Y, pa_Y;Z'})$. So that $I(\mathbf{nd_Y, pa_Y; Z}) = \min_{\mathbf{Z'}\in \mathcal{S}(Y)} I(\mathbf{nd_Y, pa_Y; Z'})$ holds.

For the necessary condition, assume that $\mathbf{Z}$ is not minimal, then there exist another sufficient statistics $\mathbf{V}$ allows $I(\mathbf{nd_Y, pa_Y;Z})> I(\mathbf{nd_Y, pa_Y;V})$ and let $\mathbf{V}: \mathcal{X}\rightarrow\mathcal{Z}$ is a function of $\mathbf{X}$ such that $\forall \mathbf{x}, \quad \mathbf{V(x)} \in\{\mathbf{z \mid z \sim Z(x)\}}$. Inspired by Fisher–Neyman factorization theorem \cite{fisher1922mathematical}, we can factorize $p(\mathbf{x})$ as below
\begin{equation}
    \forall \mathbf{x}, y \quad p(\mathbf{x} \mid y)=l_{\mathbf{Z}}^3(\mathbf{dc_y|pa_y})l_{\mathbf{Z}}^1(\mathbf{pa_y, nd_y}) l_{\mathbf{Z}}^2(\mathbf{Z(x)}, y)
\end{equation}
In above equation since $\mathbf{dc_Y}$ is decided by $\mathbf{pa_Y}$ and $Y$. We can drop the $\mathbf{dc_y}$ by defining $l_{\mathbf{Z}}^3(\mathbf{dc_y|pa_y}) \triangleq l_{\mathbf{V}}^3(\mathbf{dc_y|pa_y})$ and we can rewrite the sufficient causes condition as below
\begin{equation}
    \forall \mathbf{pa_y, nd_y}, y \quad p(\mathbf{pa_y, nd_y} \mid y)=l_{\mathbf{Z}}^1(\mathbf{pa_y, nd_y}) l_{\mathbf{Z}}^2(\mathbf{Z(x)}, y)
\end{equation}
 We define a equivalence relation $\sim$ by 
\begin{equation}
    \mathbf{z_{1}} \sim \mathbf{z_{2}} \Longleftrightarrow \frac{l_{\mathbf{Z}}^2\left(\mathbf{z_{1}}, y\right)}{l_{\mathbf{Z}}^2\left(\mathbf{z_{2}}, y\right)} \text{is a constant function of } Y
\end{equation}
There exists a sufficient cause $\mathbf{Z}^{\prime}$ such that $\mathbf{Z}$ is not a function of $\mathbf{Z}^{\prime}$. The following process proof that $\mathbf{V}$ is also sufficient cause of $Y$:

$$
\begin{aligned}
&l_{\mathbf{V}}^1(\mathbf{pa_y, nd_y})  \triangleq l_{\mathbf{Z}}^1(\mathbf{pa_y, nd_y}) \frac{l_{\mathbf{Z}}^2\left(\mathbf{Z(x)}, y\right)}{l_{\mathbf{Z}}^2\left(\mathbf{V(x)}, y\right)} \\
&l_{\mathbf{Z}}^2\left(\mathbf{V(x)}, y\right) \triangleq l_{\mathbf{V}}^2\left(\mathbf{V(x)}, y\right)
\end{aligned}
$$
Then
$$
\begin{aligned}
p(\mathbf{pa_y, nd_y} \mid y) &=l_{\mathbf{Z}}^1(\mathbf{pa_y, nd_y})l_{\mathbf{Z}}^2\left(\mathbf{Z(x)}, y\right)  \\
&=l_{\mathbf{Z}}^1(\mathbf{pa_y, nd_y}) \frac{l_{\mathbf{Z}}^2\left(\mathbf{Z(x)}, y\right)}{l_{\mathbf{Z}}^2\left(\mathbf{V(x)}, y\right)} l_{\mathbf{Z}}^2\left(\mathbf{V(x)}, y\right) \\
&=l_{\mathbf{V}}^1(\mathbf{pa_y, nd_y}) l_{\mathbf{Z}}^2\left(\mathbf{V(x)}, y\right)
\end{aligned}
$$
Since above equation holds, $\mathbf{V}$ have factorization formulation of sufficient statistics, $\mathbf{V}$ is also a sufficient statistic.
Let $\mathbf{x_{1}, x_{2}}$ such that $\mathbf{Z}^{\prime}\left(\mathbf{x_{1}}\right)=\mathbf{Z}^{\prime}\left(\mathbf{x_{2}}\right)$, then $l_{\mathbf{Z'}}^2\left(\mathbf{Z}^{\prime}\left(\mathbf{x_{1}}\right), y\right) = l_{\mathbf{Z'}}^2\left(\mathbf{Z}^{\prime}\left(\mathbf{x_{2}}\right), y\right)$
$$
\begin{aligned}
\frac{l_{\mathbf{Z}}^2\left(\mathbf{V(x_1)}, y\right)}{l_{\mathbf{Z}}^2\left(\mathbf{V(x_2)}, y\right))} 
&=\frac{p\left(\mathbf{x_{1}} \mid y\right) l_{\mathbf{Z}}^1\left(\mathbf{x_{2}}\right)}{p\left(\mathbf{x_{2}} \mid y\right) l_{\mathbf{Z}}^1\left(\mathbf{x_{1}}\right)} \\
&=\frac{l_{\mathbf{Z'}}^1\left(\mathbf{x_{1}}\right) l_{\mathbf{Z'}}^2\left(\mathbf{Z}^{\prime}\left(\mathbf{x_{1}}\right), y\right) l_{\mathbf{Z}}^1\left(\mathbf{x_{2}}\right)}{l_{\mathbf{Z}}^1\left(\mathbf{x_{1}}\right) l_{\mathbf{Z'}}^2\left(\mathbf{Z}^{\prime}\left(\mathbf{x_{2}}\right), y\right)l_{\mathbf{Z'}}^1\left(\mathbf{x_{2}}\right)} \\
&=\frac{l_{\mathbf{Z'}}^1\left(\mathbf{x_{1}}\right) l_{\mathbf{Z}}^1\left(\mathbf{x_{2}}\right)}{l_{\mathbf{Z}}^1\left(\mathbf{x_{1}}\right) l_{\mathbf{Z'}}^1\left(\mathbf{x_{2}}\right)}
\end{aligned}
$$
From the above equation we can get $\mathbf{Z(x_1)} =\mathbf{Z(x_2)}$, then we have $\mathbf{V(x_1)} =\mathbf{V(x_2)}$ because $\mathbf{V}$ is function of $\mathbf{Z'}$. Since $\mathbf{Z}$ is sufficient cause of $Y$, and $\mathbf{pa_Y, nd_Y}\perp Y| \mathbf{Z}$ in Definition \ref{def:condition_independent} holds. There exists Markov Chains $\mathbf{X-Z-V}$ and $\mathbf{(pa_Y, nd_Y)-Z-V}$. From data processing inequality, $I(\mathbf{nd_Y, pa_Y;Z})\ge I(\mathbf{nd_Y, pa_Y;V})$. 
The term  $I(\mathbf{nd_Y, pa_Y;Z})$ can be decomposed as below
\begin{equation}
\begin{split}
    I(\mathbf{nd_Y, pa_Y;Z}) =& I(\mathbf{nd_Y, pa_Y;V}) + I(\mathbf{nd_Y, pa_Y;Z | V}) \\
    \ge& I(\mathbf{nd_Y, pa_Y;V}) + I(\mathbf{nd_Y, pa_Y;Z | Z', V})\\
    =&I(\mathbf{nd_Y, pa_Y;V}) + I(\mathbf{nd_Y, pa_Y;Z | Z'})
\end{split}
\end{equation}
Since $\mathbf{Z'}$ is not the function of $\mathbf{Z}$, thus $I(\mathbf{nd_Y, pa_Y;Z | Z'})>0$, therefore we have $I(\mathbf{nd_Y, pa_Y;Z})> I(\mathbf{nd_Y, pa_Y;V})$. Thus Eq. \ref{eq:eq_in_lemma3} does not hold if $\mathbf{Z}$ is not minimal. The proof  completes.

\end{proof}
\subsection{Proof of Proposition \ref{thm:sufficient_statistics_mix}}
\begin{proof}
Under the assumption that $\mathbf{Z}$ block the path between $\mathbf{X}$ and $\mathbf{dc_Y}$, $\mathbf{X}$ and $\mathbf{dc_Y}$ are conditional independent by variable $\mathbf{Z}$. $\mathbf{X} = h(\mathbf{pa_Y, nd_Y, dc_Y}) = h(\mathbf{pa_Y, nd_Y, Z}) = h(\mathbf{pa_Y, nd_Y})$. Since all generative function of factors are invertible, we can replace $(\mathbf{nd_Y, dc_Y})$ in Markov Chain shown in the proof of Theorem \ref{thm:sufficient_statistics} by variable $\mathbf{X}$. Therefore, $p(y|\mathbf{z, pa_y, nd_y})=p(y|\mathbf{z})$ is held if and only if $p(y|\mathbf{z, x})=p(y|\mathbf{z})$ holds. Thus, under the the assumption that $\mathbf{Z}$ block the path between $\mathbf{X}$ and $\mathbf{dc_Y}$ and $h$ is a linear invertible function, the optimization processes defined in Proposition  \ref{thm:sufficient_statistics_mix} and Theorem \ref{thm:sufficient_statistics} are equivalence.
\end{proof}
\subsection{Proof of Theorem \ref{thm:sample_complexity}}
The proof follows  \cite{shamir2010learning} Theorem 3. The sketch of proof contains two steps: (i) we decompose the original objective $|I(Y;\mathbf{Z})- \hat{I}(Y;\mathbf{Z})|$ into two parts. (ii) for each part, we deduce the deterministic finite sample bound by concentration
of measure arguments on L2 norms of random vector. Let $H(X)$ denote the entropy of $X$, we have
\begin{equation}
\begin{split}
    |I(Y;\mathbf{Z})- \hat{I}(Y;\mathbf{Z})|\le|H(Y|\mathbf{Z})-\hat{H}(Y|\mathbf{Z})|+|H(Y)-\hat{H}(Y)|
\end{split}
\end{equation}
Let $\zeta(x)$ denote a continuous, monotonically increasing and concave function.
\begin{equation}
\begin{split}
    \zeta({x})= \begin{cases}0 & x=0 \\ {x} \log (1 / {x}) & 0<{x} \leq 1 / \mathrm{e} \\ 1 / \mathrm{e} & x>1 / \mathrm{e}\end{cases}
\end{split}
\end{equation}
for the term $|H(Y|Z)-\hat{H}(Y|Z)|$
\begin{equation}\label{eq:separate}
\begin{split}
    |H(Y|\mathbf{Z})-\hat{H}(Y|\mathbf{Z})|&=\left|\sum_{\mathbf{z}}(p(\mathbf{z}) H(Y \mid \mathbf{z})-\hat{p}(\mathbf{z}) \hat{H}(Y \mid \mathbf{z}))\right| \\
& \leq\left|\sum_{\mathbf{z}} p(\mathbf{z})(H(Y \mid \mathbf{z})-\hat{H}(Y \mid \mathbf{z}))\right|+\left|\sum_{\mathbf{z}}(p(\mathbf{z})-\hat{p}(\mathbf{z})) \hat{H}(Y \mid \mathbf{z})\right|
\end{split}
\end{equation}
For the  first summand in this bound, we introduce variable $\bm{\epsilon}$ to help decompose $p(y|\mathbf{z})$, where $\bm{\epsilon}$ is independent with the parents $\mathbf{pa_y}$ (i.e. $\bm{\epsilon}\perp\mathbf{pa_y}$)
\begin{equation}
\begin{split}
    &\left|\sum_{\mathbf{z}} p(\mathbf{z})(H(Y \mid \mathbf{z})-\hat{H}(Y \mid \mathbf{z}))\right|\\
    & \leq\left|\sum_{\mathbf{z}} p(\mathbf{z}) \sum_{{y}}(\hat{p}({y} \mid \mathbf{z}) \log (\hat{p}({y} \mid \mathbf{z}))-p(y \mid \mathbf{z}) \log (p(y \mid \mathbf{z})))\right| \\
& \leq \sum_{\mathbf{z}} p(\mathbf{z}) \sum_{y} \zeta(|\hat{p}(y \mid \mathbf{z})-p(y \mid \mathbf{z})|) \\
&=\sum_{\mathbf{z}} p(\mathbf{z}) \sum_{y} \zeta\left(\left|\sum_{\bm{\epsilon}} p(\bm{\epsilon} \mid \mathbf{z})(\hat{p}(y \mid \mathbf{z}, \bm{\epsilon})-p(y \mid \mathbf{z}, \bm{\epsilon}))\right|\right) \\
&=\sum_{\mathbf{z}} p(\mathbf{z}) \sum_{y} \zeta(\|\hat{\mathbf{p}}(y \mid \mathbf{z}, \bm{\epsilon})-\mathbf{p}(y \mid \mathbf{z}, \bm{\epsilon})\| \sqrt{V(\mathbf{p}(\bm{\epsilon} \mid \mathbf{z}))})\\
\end{split}
\end{equation}
where $\frac{1}{m}V(x)$ denote the variance of vector $x$. For the second summand in Eq.\ref{eq:separate}.
\begin{equation}
\begin{split}
    \left|\sum_{\mathbf{z}}(p(\mathbf{z})-\hat{p}(z)) \hat{H}(Y \mid z)\right| \leq\|\mathbf{p}(\mathbf{z})-\hat{\mathbf{p}}(\mathbf{z})\| \cdot \sqrt{V(\hat{\mathbf{H}}(Y \mid \mathbf{z}))}
\end{split}
\end{equation}
For the summand $|H(Y)-\hat{H}(Y)|$:
\begin{equation}
\begin{split}
|H(Y)-\hat{H}(Y)| &=\left|\sum_{y} p(y) \log (p(y))-\hat{p}(y) \log (\hat{p}(y))\right| \\
& \leq \sum_{y} \zeta(|p(y)-\hat{p}(y)|) \\
&=\sum_{y} \zeta\left(\left|\sum_{\mathbf{z}} \sum_{\bm{\epsilon}}p(\bm{\epsilon} \mid \mathbf{z})(p(\mathbf{z})p( y|\bm{\epsilon})-\hat{p}(\mathbf{z})p( y|\bm{\epsilon}))\right|\right) \\
& \leq \sum_{y} \zeta(\|\mathbf{p}(\mathbf{z})p(y|\bm{\epsilon})-\hat{\mathbf{p}}(\mathbf{z})p(y|\bm{\epsilon})\| \sqrt{V(\mathbf{p}(\bm{\epsilon} \mid \mathbf{z}))})
\end{split}
\end{equation}

Combining above bounds:
\begin{equation}
\begin{split}
    |I(Y;\mathbf{Z})- \hat{I}(Y;\mathbf{Z})|\le&\sum_{y} \zeta(\|\mathbf{p}(\mathbf{z}, y|\bm{\epsilon})-\hat{\mathbf{p}}(\mathbf{z}, y|\bm{\epsilon})\|\sqrt{V(\mathbf{p}(\bm{\epsilon} \mid \mathbf{z}))}) \\
    &+\sum_{\mathbf{z}} p(\mathbf{z}) \sum_{y} \zeta(\|\hat{\mathbf{p}}(y \mid \mathbf{z}, \epsilon)-\mathbf{p}(y \mid \mathbf{z}, \bm{\bm{\epsilon}})\| \sqrt{V(\mathbf{p}(\bm{\bm{\epsilon}} \mid z))})\\
    &+\|\mathbf{p}(\mathbf{z})-\hat{\mathbf{p}}(\mathbf{z})\| \cdot \sqrt{V(\hat{\mathbf{H}}(Y \mid \mathbf{z}))}
\end{split}
\end{equation}
Let $\bm{\rho}$ be a distribution vector of arbitrary cardinality, and let $\hat{\bm{\rho}}$ be an empirical estimation of $\bm{\rho}$ based on a sample of size $m$. Then the error $\|\bm{\rho}-\hat{\bm{\rho}}\|$ will be bounded  with a probability of at least $1-\delta$
\begin{equation}\label{eq:delta_}
\begin{split}
\|\bm{\rho}-\hat{\bm{\rho}}\| \leq \frac{2+\sqrt{2 \log (1 / \delta)}}{\sqrt{m}}
\end{split}
\end{equation}
Following the proof of Theorem 3 in \cite{shamir2010learning}, to make sure the bounds hold over $|\mathcal{Y}| + 2|$  quantities, we replace $\delta$ in Eq.\ref{eq:delta_} by $\delta/(|\mathcal{Y}| + 2|$, than substitute $\|\mathbf{p}(\mathbf{z}, y|\bm{\epsilon})-\hat{\mathbf{p}}(\mathbf{z}, y|\bm{\epsilon})\|$ $\|\hat{\mathbf{p}}(y \mid \mathbf{z}, \bm{\epsilon})-\mathbf{p}(y \mid \mathbf{z}, \bm{\epsilon})\|, \|\mathbf{p}(\mathbf{z})-\hat{\mathbf{p}}(\mathbf{z})\|, $ by Eq.\ref{eq:delta_}.
\begin{equation}\label{eq:i_y_z}
\begin{split}
    |I(Y ; \mathbf{Z})-\hat{I}(Y ; \mathbf{Z})| \leq &(2+\sqrt{2 \log ((|\mathcal{Y}|+2) / \delta)}) \sqrt{\frac{V(\hat{\mathbf{H}}(Y \mid \mathbf{z}))}{m}} \\
    &+2 |\mathcal{Y}| h\left(2+\sqrt{2 \log ((|\mathcal{Y}|+2) / \delta)} \sqrt{\frac{V(\mathbf{p}(\bm{\epsilon} \mid\mathbf{z}))}{m}}\right)\\
\end{split}
\end{equation}
There exist a constant $C$, where $2+\sqrt{2 \log ((|\mathcal{Y}|+2) / \delta)}\le\sqrt{C\log ((|\mathcal{Y}|) / \delta)}$. 
From the fact that variance of any random variable bounded in [0, 1] is at most 1/4, we analyze the bound under two different cases:

\textbf{In general case} ($\mathbf{z}=\phi(\mathbf{x})$ is arbitrary representation of $\mathbf{x}$), 
\begin{equation}
\begin{split}
    V(\mathbf{p}(\bm{\epsilon} \mid \mathbf{z}))\le \frac{|\mathcal{Z}|}{4}
\end{split}
\end{equation}
let $m$ denote the number of sample, we get a lower bound of $m$, which is also known as sample complexity.
\begin{equation}
\begin{split}
    m \geq \frac{C}{4} \log (|\mathcal{Y}| / \delta)|\mathcal{Z}| \mathrm{e}^{2}
\end{split}
\end{equation}

\textbf{In ideal case}($\mathbf{z}$ is sufficient cause of $\mathbf{x}$) in that case $\mathbf{z}$ is independent with the exogenous noise $\bm{\epsilon}$, $\mathbf{z}\perp\bm{\epsilon}$:
\begin{equation}
\begin{split}
    V(\mathbf{p}(\bm{\epsilon} \mid \mathbf{z}))\le \beta
\end{split}
\end{equation}
\begin{equation}
\begin{split}
    m \geq \frac{C}{4} \log (|\mathcal{Y}| / \delta)|\beta| \mathrm{e}^{2}
\end{split}
\end{equation}

\begin{equation}
    \sqrt{\frac{C \log (|\mathcal{Y}| / \delta) V(\mathbf{p}(\epsilon \mid \mathbf{z}))}{m}} \leq \sqrt{\frac{C \log (|\mathcal{Y}| / \delta)|\mathcal{Z}|}{4m}} \leq 1 / \mathrm{e}
\end{equation}
Then, from the fact that (\cite{shamir2010learning}):
\begin{equation}
\begin{split}
h\left(\sqrt{\frac{\nu}{m}}\right) &=\left(\sqrt{\frac{\nu}{m}} \log \left(\sqrt{\frac{m}{v}}\right)\right) \\
& \leq \frac{\sqrt{v} \log (\sqrt{m})+1 / \mathrm{e}}{\sqrt{m}},
\end{split}
\end{equation}
We can get the upper bound of second summand in Eq.\ref{eq:i_y_z} as follows
\begin{equation}\label{eq:yh}
\begin{split}
    &\sum_y h\left(\sqrt{C\log ((|\mathcal{Y}|) / \delta)} \sqrt{\frac{V(\mathbf{p}(\epsilon \mid z))}{m}}\right)\\
    \le& \frac{\sqrt{C \log (|\mathcal{Y}| / \delta)} \log (m)\left(|\mathcal{Y}|\sqrt{V(\mathbf{p}(\bm{\epsilon} \mid \mathbf{z}))}\right)+\frac{2}{\mathrm{e}}|\mathcal{Y}|}{2 \sqrt{m}}
\end{split}
\end{equation}
\textbf{In general case}:
\begin{equation}
\begin{split}
    Eq.\ref{eq:yh} \le \frac{\sqrt{C \log (|\mathcal{Y}| / \delta)} \log (m)\left(|\mathcal{Y}|\sqrt{\mathcal{|Z|}}\right)+\frac{2}{\mathrm{e}}|\mathcal{Y}|}{2 \sqrt{m}}
\end{split}
\end{equation}

\textbf{In ideal case}:
\begin{equation}
\begin{split}
    Eq.\ref{eq:yh} \le \frac{\sqrt{C \log (|\mathcal{Y}| / \delta)} \log (m)\left(|\mathcal{Y}|\sqrt{\beta}\right)+\frac{2}{\mathrm{e}}|\mathcal{Y}|}{2 \sqrt{m}}
\end{split}
\end{equation}
For the first summand in Eq.\ref{eq:i_y_z}, we follow the fact (\cite{shamir2010learning} Theorem 3) that:
\begin{equation}
\begin{split}
    V(\mathbf{H}(Y \mid \mathbf{z})) \leq \frac{|Z| \log ^{2}(|\mathcal{Y}|)}{4}
\end{split}
\end{equation}
Finally we accomplish the proof of Theorem \ref{thm:sample_complexity}.

\subsection{Extension of Theorem \ref{thm:sample_complexity} for distribution shift}

\begin{proof}
The risk under target domain is defined as $|I_\mathcal{T}(Y ; \mathbf{Z})-\hat{I}_\mathcal{T}(Y ; \mathbf{Z})| $, the proof is start by the following equation shown in the proof of Theorem \ref{thm:sample_complexity}.
\begin{equation}\label{eq:i_y_z}
\begin{split}
    |I_\mathcal{T}(Y ; \mathbf{Z})-\hat{I}_\mathcal{T}(Y ; \mathbf{Z})| \leq &(2+\sqrt{2 \log ((|\mathcal{Y}|+2) / \delta)}) \sqrt{\frac{V(\hat{\mathbf{H}}_\mathcal{T}(Y \mid \mathbf{z}))}{m}} \\
    &+2 |\mathcal{Y}| h\left(2+\sqrt{2 \log ((|\mathcal{Y}|+2) / \delta)} \sqrt{\frac{V(\mathbf{p}(\bm{\epsilon} \mid\mathbf{z}))}{m}}\right)\\
\end{split}
\end{equation}
We will then bound the term  $V(\hat{\mathbf{H}}_\mathcal{T}(Y \mid \mathbf{z}))$ by the variance of entropy on source data. From the definition of function $V$, we have
\begin{equation}\label{eq:v_h_y|z}
\begin{split}
    \sqrt{V(\hat{\mathbf{H}}(Y \mid \mathbf{z}))} &\leq \sqrt{\sum_{y}(\hat{H}(Y \mid \mathbf{z})-\hat{H}(Y))^{2}}+\sqrt{\sum_{y}\left(\hat{H}(Y)-\frac{1}{|\mathcal{Z}|} \sum_{z^{\prime}} \hat{H}\left(Y \mid \mathbf{z}^{\prime}\right)\right)^{2}}\\
    &\le \left(1+\frac{1}{\sqrt{|\mathcal{Z}|}}\right)\left|\sum_{z^{\prime}}\left(\hat{H}(Y)-\hat{H}\left(Y \mid \mathbf{z}^{\prime}\right)\right)\right|\\
    &=\left(1+\frac{1}{\sqrt{|\mathcal{Z}|}}\right)\frac{1}{\min _\mathbf{z} p(\mathbf{z})}\left(\hat{H}(Y)-\sum_{\mathbf{z}} p(y) \hat{H}(Y\mid \mathbf{z})\right)\\
    &=\left(1+\frac{1}{\sqrt{|\mathcal{Z}|}}\right)\frac{1}{\min _{\mathbf{z}} p(\mathbf{z})} \hat{I}(\mathbf{Z} ; Y) 
\end{split}
\end{equation}
Supposing that only source data $\mathcal{S}(\mathbf{X}, Y)$ is available, for the term $\hat{I}_\mathcal{T}(\mathbf{Z} ; Y) $ evaluated on target dataset, we change the measure by important sampling and Jensen's inequality. The way helps bound $\hat{I}_\mathcal{T}(\mathbf{Z} ; Y)$ by the evaluation on source domain. Denoting $D_{KL}(P||Q)$ by the Kullback-Leibler divergence between distribution $P$ and $Q$. $\mathcal{S}(\mathbf{z}, y)$ and $\mathcal{T}(\mathbf{z}, y)$ are the distribution of $p(\mathbf{z}, Y)$ on source domain and target domain separately.
\begin{equation}
\begin{split}
    \hat{I}_\mathcal{T}(\mathbf{Z}, Y) &= \mathbb{E}_{\mathcal{T}(\mathbf{z}, y)} \log \frac{\mathbf{p}(\mathbf{z}, y)}{\mathbf{p}(\mathbf{z}) \mathbf{p}(y)}\\
    &\le D_{KL}(\mathcal{T}(\mathbf{z}, y) || \mathcal{S}(\mathbf{z}, y)) + \log \mathbb{E}_{\mathcal{S}(\mathbf{z}, y)} \frac{\hat{p}(\mathbf{z}, y)}{\hat{p}(\mathbf{z}) \hat{p}(y)}
\end{split}
\end{equation}
Substituting $\hat{I}_\mathcal{T}(\mathbf{Z}, Y)$ into Eq. \ref{eq:v_h_y|z} and Eq. \ref{eq:i_y_z}. Since $|\mathcal{Z}|>1$, let $D = \frac{2}{\min _{\mathbf{z}} p(\mathbf{z})}$ and $I_\mathcal{S} = \mathbb{E}_{\mathcal{S}(\mathbf{z}, y)} \frac{\hat{p}(\mathbf{z}, y)}{\hat{p}(\mathbf{z}) \hat{p}(y)}$. We can get the bounds under two different cases:

\textbf{In general case}, since since $\mathbf{Z}$ is arbitrary representation of $\mathbf{X}$, we get $D_{KL}(\mathcal{T}(\mathbf{z}, y) || \mathcal{S}(\mathbf{z}, y)) > 0$. We cannot drop the $D_{KL}$ term. Thus we have the bound:
\begin{equation}
    |I_\mathcal{T}(Y ; \mathbf{Z})-\hat{I}_\mathcal{T}(Y ; \mathbf{Z})| \leq
    \frac{\sqrt{C \log (|\mathcal{Y}| / \delta)}\left(|\mathcal{Y}|\sqrt{|\mathcal{Z}|} \log (m)+ D_{KL}(\mathcal{T}||\mathcal{S}) + D I_\mathcal{S})\right)+\frac{2}{e}|\mathcal{Y}|}{\sqrt{m}}
\end{equation}

\textbf{In ideal case}, since $\mathbf{Z}$ is sufficient cause of $Y$, we get $D_{KL}(\mathcal{T}(\mathbf{z}, y) || \mathcal{S}(\mathbf{z}, y)) = 0$ from Assumption \ref{asp:distribution_shift}.
\begin{equation}
    |I_\mathcal{T}(Y ; \mathbf{Z})-\hat{I}_\mathcal{T}(Y ; \mathbf{Z})| \leq
    \frac{\sqrt{C \log (|\mathcal{Y}| / \delta)}\left(|\mathcal{Y}|\sqrt{|\beta|} \log (m)+  D I_\mathcal{S})\right)+\frac{2}{e}|\mathcal{Y}|}{\sqrt{m}}
\end{equation}
\end{proof}

\section{Experimental Details} \label{sup:exp_details}

\subsection{Model Architecture and Implementation Details}
The hyper-parameters are determined by grid search.
Specifically, the learning rate and batch size are tuned in the ranges of {{$[10^{-4}, 10^{-1}]$} and $[64,128,256,512,1024]$}, respectively.
The weighting parameter $\lambda$ is tuned in $[0.001]$. Perturbation degrees are set to be $\beta=\{0.1, 0.2, 0.1, 0.3\}$ for Coat, Yahoo!R3, PCIC and CPC separately. The representation dimension is empirically set as 64. All the experiments are conducted based on a server with a 16-core CPU, 128g memories and an RTX 5000 GPU. The deep model architecture is shown as follows:

(1)Representation learning method $\phi(\mathbf{x})$:
If the dataset is Yahoo!R3 or PCIC, in which only the user id and item id are the input, we first use an embedding layer. The representation function architecture is:
\begin{itemize}
    \item Concat(Embedding(user id, 32), Embedding(item id, 32))
    \item Linear(64, 64), ELU()
    \item Linear(64, representation dim), ELU()
\end{itemize}
Then for the dataset Coat and CPC, the feature dimension is 29 and 47 separately. It do not use embedding layer at first. The representation function architecture is.
\begin{itemize}
    \item Linear(64, 64), ELU()
    \item Linear(64, representation dim), ELU()
\end{itemize}
(2)Downstream Prediction Model $g(\mathbf{z})$:
\begin{itemize}
\item Linear(representation dim, 64), ELU()
\item Linear(64, 2)
\end{itemize}
\begin{figure*}[h]
\begin{center}
\centerline{\includegraphics[width=0.6\columnwidth]{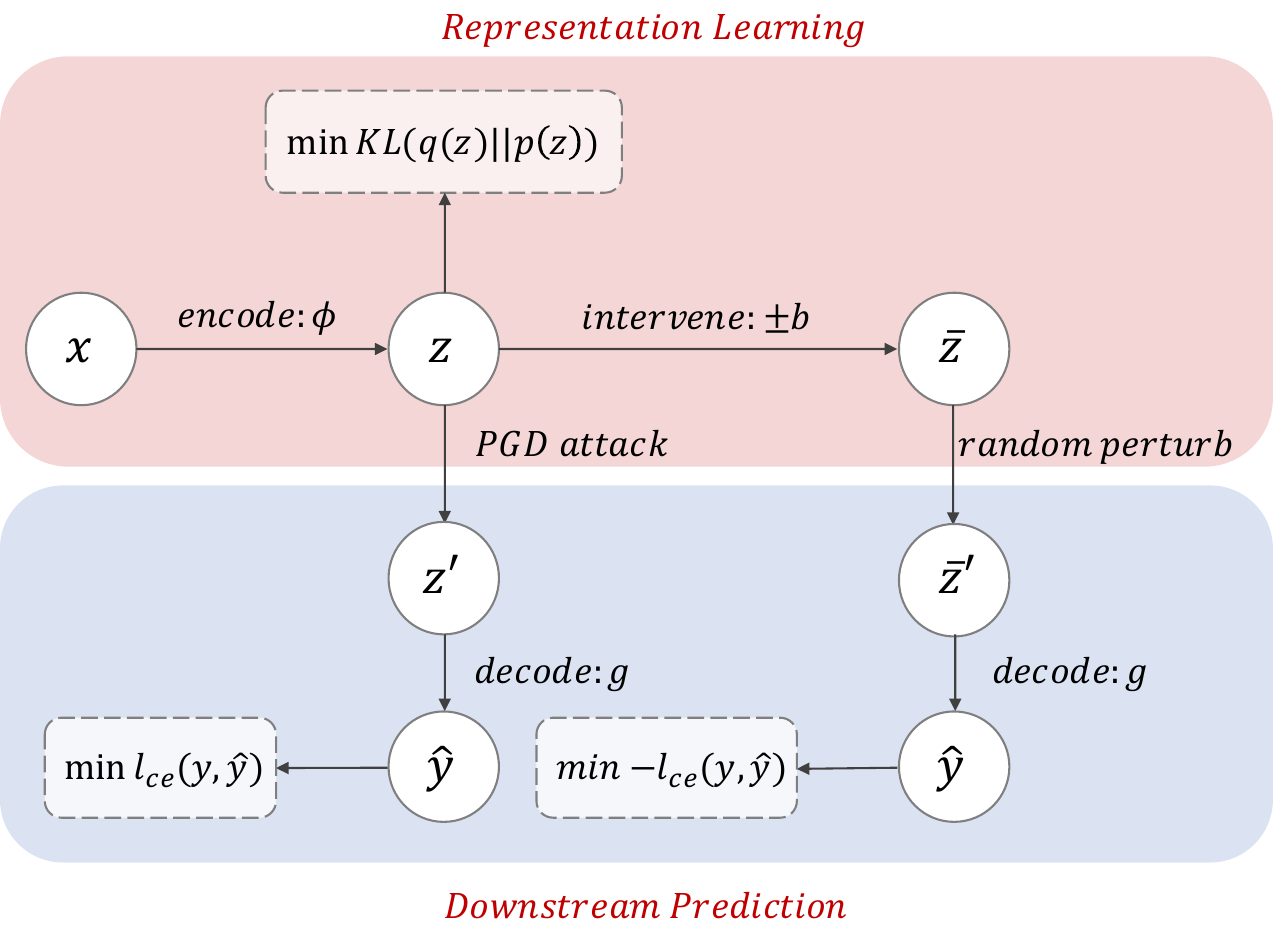}}
\caption{The figure demonstrates the model architecture of CaRI}
\label{fig:archi}
\end{center}
\end{figure*}

\section{Additional Results} \label{sup:exp_results}
Due to the page limit in main text, we present the additional test results and analysis in this section. Table \ref{tab:pcic_std} overall results with standard error via 5 tims runs.
Fig. \ref{fig:6} and \ref{fig:7} compare the distance correlation metric given by the training under standard and robust mode. It shows that our method performs consistently better compared with base methods in both modes, with a higher distance correlation, under smaller variance. The gap is obvious especially in the learning of parental information, which is the main focus of our approaches.

Fig. \ref{fig:8} and \ref{fig:9} record the results along optimization process and until convergence, under different settings of the pertubation degree $\beta$, considering the dataset CelebA-anno. The annotation smile is used as the label to be prediced, and other features are the source data.  It shows that when the optimization process is not finished, both approaches have similar performance, with unstability evidenced by large variance of the DC metric. However, our method outperforms the baseline when the optimization converges, owning a higher DC with smaller variations.  The results also show that $\beta$ is an important factor for training the model. Larger $\beta$ often leads to higher variance of the training of the model.

Fig.\ref{fig:epsilon_result} demonstrates how robust training  degree ($\beta=\{0.1, 0.3, 0.5, 0.7, 1.0\}$) influences the downstream prediction under adversarial settings. We conduct the experiments on the attacked real-world dataset by PGD attacker. From Fig.\ref{fig:epsilon_result}, we find that our method is better than base method, because the base model's ability on standard prediction is broken by adversarial training. When $\beta$ is small, our method behaves closely to the r-CVAE in all the datasets. When $\beta$ gets larger, the difference between performance of CaRI and that of r-CVAE continuously enlarges in Yahoo!R3. In PCIC, the gap becomes the largest among all when $\beta=0.5$, and narrows down to 0 when $\beta=0.7$.  This is because in our framework, we explicitly deploy a model to achieve more robust representations, while others fail. 
\begin{figure*}[h]
\begin{center}
\centerline{\includegraphics[width=0.8\columnwidth]{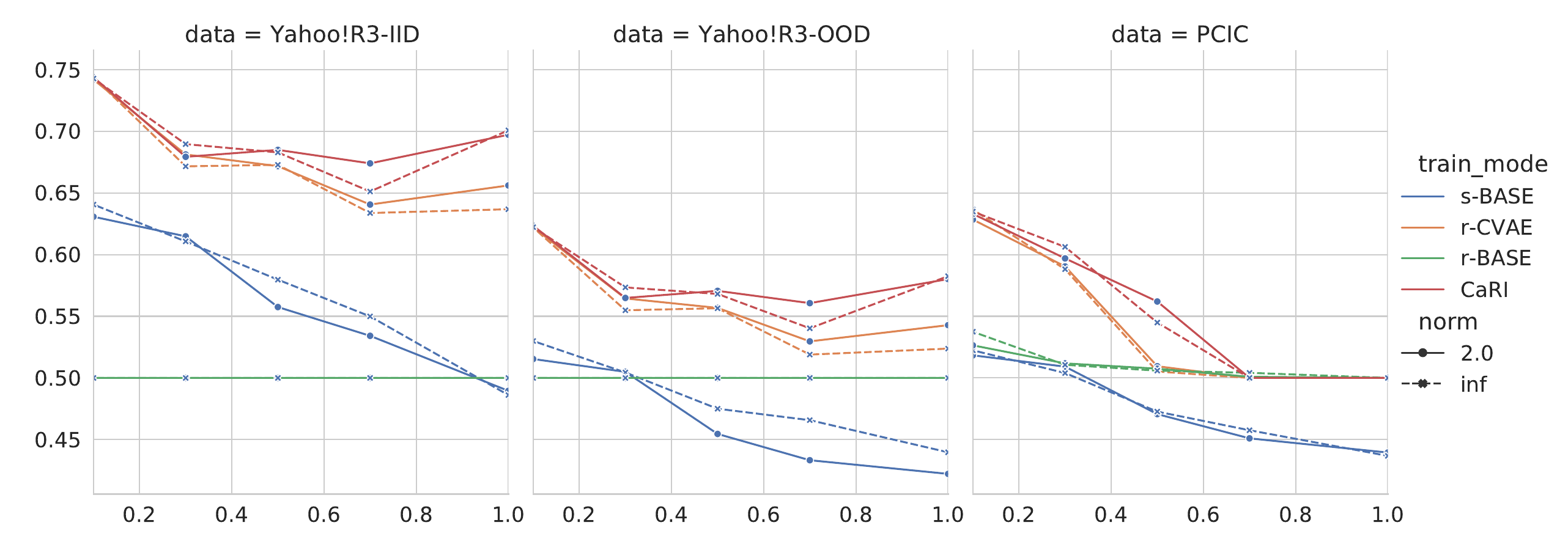}}
\caption{ Results under  different adversarial perturbations $\beta$ on three datasets. Axis-x is the attack degree $\beta$. Axis-y is the adv-AUC under attacked test datasets.}
\label{fig:epsilon_result}
\end{center}
\vspace{-0mm}
\end{figure*}

% \begin{table*}
% \caption{Overall Results on CPC and PCIC}
%     \center
% \small
% \renewcommand\arraystretch{1.1}
% \setlength{\tabcolsep}{4.8pt}
% \begin{threeparttable}  
% \scalebox{.95}{
%     \begin{tabular}{c|c|cc|cc|cc|cc}
%     \hline\hline
%     Dataset& Method&\multicolumn{4}{c|}{p=$\infty$}&\multicolumn{4}{c}{p=2} \\ \hline
%         & Metrics & AUC & ACC & advAUC & advACC & AUC & ACC & advAUC & advACC \\ \hline
%         \multirow{7}{*}{CPC}&base(robust) & 0.5 & 0.1017 & 0.5 & 0.1017 & 0.5 & 0.1012 & 0.5 & 0.0928 \\ 
%         &base(standard) & 0.7144 & 0.7076 & 0.4506 & 0.4305 & 0.7103 & 0.7612 & 0.4489 & 0.4017 \\ 
%         &IB(standard) & 0.7113 & 0.7265 & 0.5776 & 0.6952 & 0.7124 & 0.7066 & 0.5647 & 0.6642 \\
%         &r-CVAE(robust) & 0.7182 & 0.7749 & 0.6694 & 0.7196 & {0.7139} & 0.7574 & 0.6611 & 0.7426 \\ 
%         &r-CVAE(standard) & 0.7183 & 0.7703 & 0.6467 & 0.6809 & 0.7124 & 0.7692 & 0.6402 & 0.691 \\ 
%         &CaRI(robust) & \textbf{0.7271} & \textbf{0.7834} & \textbf{0.6851} & \textbf{0.7501} & \textbf{0.7224} & \textbf{0.7874} & \textbf{0.6776} & \textbf{0.7607} \\ 
%         &CaRI(standard) & 0.7225 & 0.7543 & 0.6606 & 0.6995 & 0.7175 & 0.778 & 0.6661 & 0.7378 \\ 
%         \hline
        
%         \hline\hline
%     \end{tabular}
% }   
% \end{threeparttable}    
% \label{tab:overall}
% \end{table*}

\begin{table*}[h]
    \center
\caption{Additional overall results with standard error.}
\renewcommand\arraystretch{1.1}
\setlength{\tabcolsep}{3.3pt}
\begin{threeparttable}
\scalebox{.95}{
    \begin{tabular}{c|c|c|c|cc|cc|cc|cc}
    \hline\hline
dataset &     &     &      & AUC    & std    & ACC    & std    & adv\_AUC & std    & adv\_ACC & std    \\
\hline
\multirow{8}{*}{PCIC} & \multirow{4}{*}{standard} & \multirow{2}{*}{p=2}   & CaRI & 0.6416 & 0.0078 & 0.6803 & 0.0014 & 0.619    & 0.004  & 0.6625   & 0.0041 \\
         & &     & r-CVAE & 0.6328 & 0.0023 & 0.6725 & 0.0042 & 0.5893   & 0.0419 & 0.6429   & 0.0201 \\
         & &  \multirow{2}{*}{p=$\infty$} & CaRI & 0.6447 & 0.0041 & 0.6817 & 0.0043 & 0.6148   & 0.011  & 0.664    & 0.0104 \\
         & &     & r-CVAE & 0.6358 & 0.014  & 0.6779 & 0.0066 & 0.6138   & 0.0062 & 0.6601   & 0.0048 \\
&\multirow{4}{*}{robust}   & \multirow{2}{*}{p=2}   & CaRI & 0.6363 & 0.0045 & 0.6709 & 0.0042 & 0.6332   & 0.0024 & 0.6576   & 0.0006 \\
         & &     & r-CVAE & 0.63   & 0.0075 & 0.674  & 0.0069 & 0.6187   & 0.0051 & 0.6493   & 0.0013 \\
         & & \multirow{2}{*}{p=$\infty$} & CaRI & 0.639  & 0.007  & 0.6761 & 0.0024 & 0.6225   & 0.0057 & 0.6638   & 0.001  \\
         & &     & r-CVAE & 0.6363 & 0.0066 & 0.6733 & 0.0058 & 0.6088   & 0.0098 & 0.6596   & 0.0124\\
         \hline
         \multirow{8}{*}{Yahoo!R3 OOD}&\multirow{4}{*}{standard} & \multirow{2}{*}{p=2}   & CaRI & 0.6276 & 0.0001 & 0.6255 & 0.0022 & 0.5917 & 0.0071 & 0.5917 & 0.0072 \\
         & &     & r-CVAE & 0.6233 & 0.0005 & 0.6243 & 0.002  & 0.5865 & 0.0022 & 0.5872 & 0.0025 \\
         
         & &  \multirow{2}{*}{p=$\infty$} & CaRI & 0.629  & 0.0011 & 0.6257 & 0.0002 & 0.5966 & 0.0049 & 0.5965 & 0.0042 \\
         & &     & r-CVAE & 0.6253 & 0.0023 & 0.6249 & 0.0014 & 0.5855 & 0.0016 & 0.5863 & 0.0019 \\
& \multirow{4}{*}{robust}   & \multirow{2}{*}{p=2}   & CaRI & 0.6242 & 0.0009 & 0.6307 & 0.0012 & 0.6008 & 0.0009 & 0.601  & 0.0016  \\
         & &     & r-CVAE & 0.6191 & 0.0013 & 0.6241 & 0.0051 & 0.5882 & 0.0014 & 0.5907 & 0.0009 \\
         & & \multirow{2}{*}{p=$\infty$} & CaRI & 0.6238 & 0.0011 & 0.6284 & 0.0017 & 0.5993 & 0.0019 & 0.5999 & 0.0026  \\
         & &     & r-CVAE & 0.6186 & 0.001  & 0.6235 & 0.0028 & 0.5886 & 0.0014 & 0.5912 & 0.0012\\
         \hline
         \multirow{8}{*}{Yahoo!R3 i.i.d.}&\multirow{4}{*}{standard} & \multirow{2}{*}{p=2}   & CaRI & 0.7493 & 0.0004 & 0.7495 & 0.0015 & 0.7188 & 0.0015 & 0.7072 & 0.0013 \\
         & &     & r-CVAE & 0.7487 & 0.0001 & 0.7529 & 0.0027 & 0.7202 & 0.0029 & 0.7099 & 0.0027 \\
         
         & &  \multirow{2}{*}{p=$\infty$} & CaRI & 0.7497 & 0.0004 & 0.7503 & 0.0019 & 0.7191 & 0.0023 & 0.7099 & 0.0026 \\
         & &     & r-CVAE & 0.7488 & 0.0001 & 0.7515 & 0.0008 & 0.7191 & 0.0021 & 0.7072 & 0.0015 \\
& \multirow{4}{*}{robust}   & \multirow{2}{*}{p=2}   & CaRI & 0.7374 & 0.0024 & 0.7158 & 0.0061 & 0.7247 & 0.0026 & 0.7159 & 0.0036  \\
         & &     & r-CVAE & 0.7376 & 0.0018 & 0.7151 & 0.0045 & 0.7194 & 0.0020 & 0.7082 & 0.0021 \\
         & & \multirow{2}{*}{p=$\infty$} & CaRI & 0.7378 & 0.0015 & 0.7168 & 0.0015 & 0.7210 & 0.0031 & 0.7107 & 0.0040  \\
         & &     & r-CVAE & 0.7341 & 0.0007 & 0.7093 & 0.0035 & 0.7180 & 0.0017 & 0.7080 & 0.0016\\
         \hline
         \multirow{8}{*}{Coat OOD}&\multirow{4}{*}{standard} & \multirow{2}{*}{p=2}   & CaRI & 0.5725 & 0.0005 & 0.5732 & 0.0005 & 0.5608 & 0.0003 & 0.5601 & 0.0004 \\
         & &     & r-CVAE & 0.5671 & 0.0005 & 0.5649 & 0.0006 & 0.5586 & 0.0002 & 0.554  & 0.0001 \\
         
         & &  \multirow{2}{*}{p=$\infty$} & CaRI & 0.5705 & 0.0013 & 0.5718 & 0.0017 & 0.5643 & 0.0001 & 0.5659 & 0.0006 \\
         & &     & r-CVAE & 0.5656 & 0.0005 & 0.5643 & 0.0007 & 0.5527 & 0.0074 & 0.5478 & 0.0081 \\
& \multirow{4}{*}{robust}   & \multirow{2}{*}{p=2}   & CaRI & 0.5705 & 0.0015 & 0.5675 & 0.0015 & 0.5674 & 0.0002 & 0.565  & 0.0012  \\
         & &     & r-CVAE & 0.5634 & 0.0014 & 0.5591 & 0.0018 & 0.5572 & 0.0009 & 0.5522 & 0.0003 \\
         & & \multirow{2}{*}{p=$\infty$} & CaRI & 0.5707 & 0.0017 & 0.5681 & 0.0024 & 0.5653 & 0.0019 & 0.5659 & 0.0011  \\
         & &     & r-CVAE & 0.5629 & 0.0017 & 0.5586 & 0.0028 & 0.559  & 0.0004 & 0.5544 & 0.0007\\
         \hline
         \multirow{8}{*}{Coat i.i.d.}&\multirow{4}{*}{standard} & \multirow{2}{*}{p=2}   & CaRI & 0.7248 & 0.0011 & 0.7305 & 0.0016 & 0.7069 & 0.0023 & 0.7125 & 0.0036 \\
         & &     & r-CVAE & 0.7129 & 0.0009 & 0.7206 & 0.0022 & 0.7023 & 0.0041 & 0.7059 & 0.0061 \\
         
         & &  \multirow{2}{*}{p=$\infty$} & CaRI & 0.7283 & 0.0013 & 0.7355 & 0.0015 & 0.7125 & 0.0007 & 0.7196 & 0.001 \\
         & &     & r-CVAE & 0.7106 & 0.0029 & 0.7184 & 0.0033 & 0.7029 & 0.0008 & 0.7106 & 0.0094 \\
& \multirow{4}{*}{robust}   & \multirow{2}{*}{p=2}   & CaRI & 0.7265 & 0.0032 & 0.7331 & 0.0027 & 0.7196 & 0.0046 & 0.7261 & 0.0042  \\
         & &     & r-CVAE & 0.7087 & 0.0005 & 0.7169 & 0.0016 & 0.7058 & 0.002  & 0.7141 & 0.0036 \\
         & & \multirow{2}{*}{p=$\infty$} & CaRI & 0.7276 & 0.0028 & 0.7339 & 0.002  & 0.7208 & 0.0023 & 0.727  & 0.0019  \\
         & &     & r-CVAE & 0.7147 & 0.0023 & 0.7222 & 0.0026 & 0.7105 & 0.0039 & 0.7181 & 0.0043 \\
         \hline\hline
    \end{tabular}
}   
\end{threeparttable}    
\label{tab:pcic_std}

\end{table*}

\begin{figure*}[h]
\begin{center}
\centerline{\includegraphics[width=0.8\columnwidth]{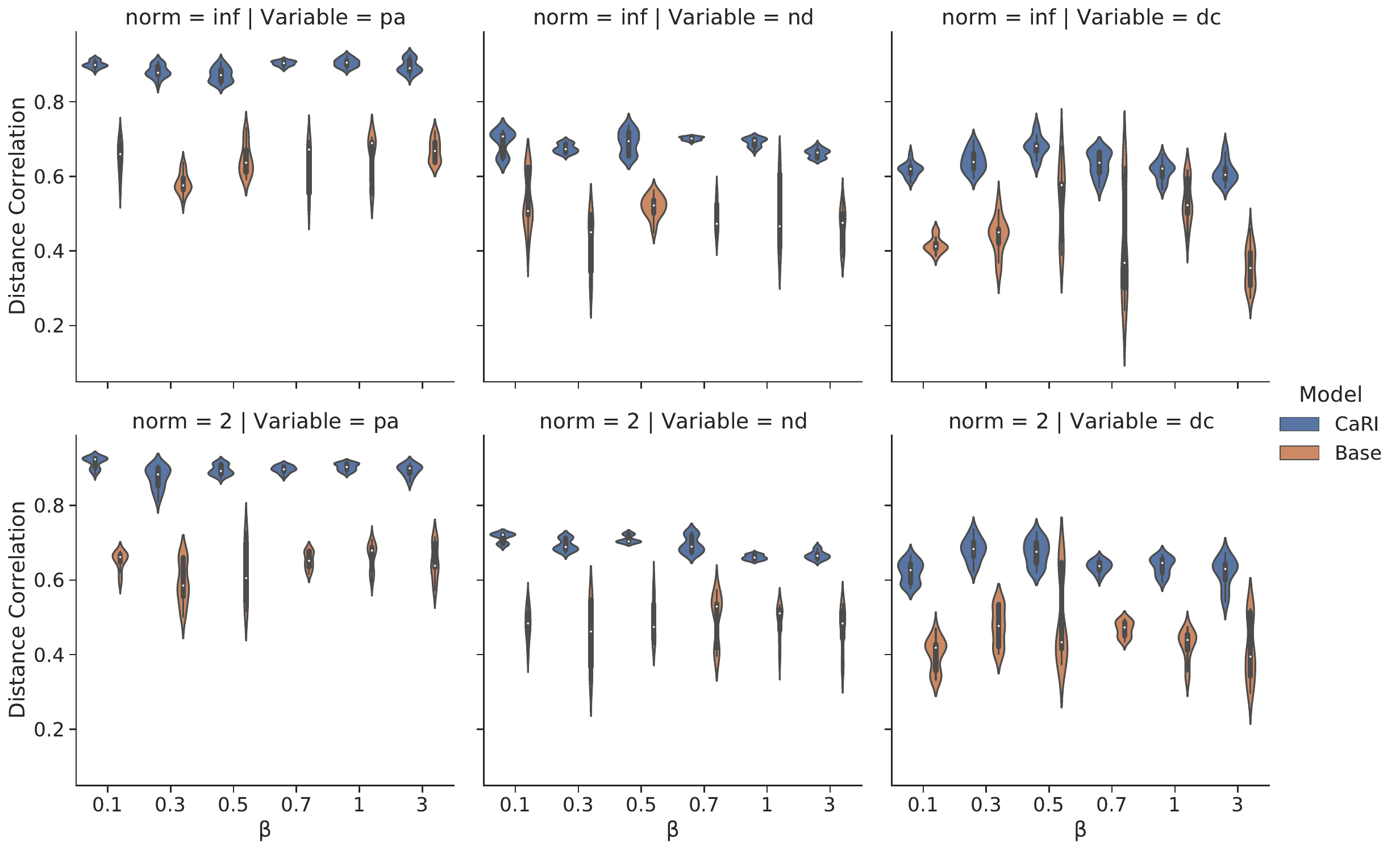}}
\caption{Identify results on synthetic dataset over different range of $\beta$ under robust training.}
\label{fig:6}
\end{center}
\end{figure*}

\begin{figure*}[h]
\begin{center}
\centerline{\includegraphics[width=0.8\columnwidth]{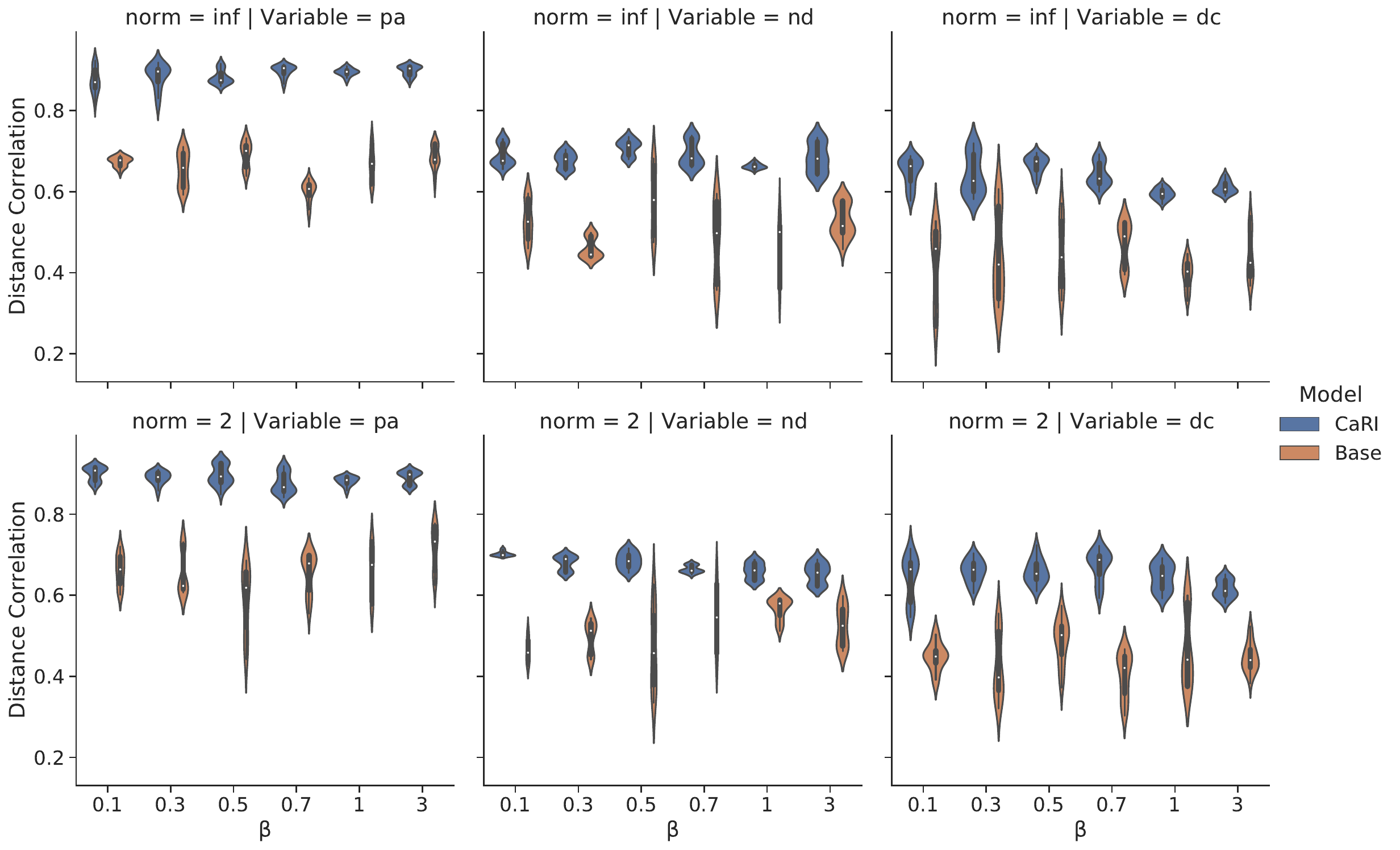}}
\caption{Identify results on synthetic dataset over different range of $\beta$ under standard training.}
\label{fig:7}
\end{center}
\end{figure*}

\begin{figure*}[h]
\begin{center}
\centerline{\includegraphics[width=0.8\columnwidth]{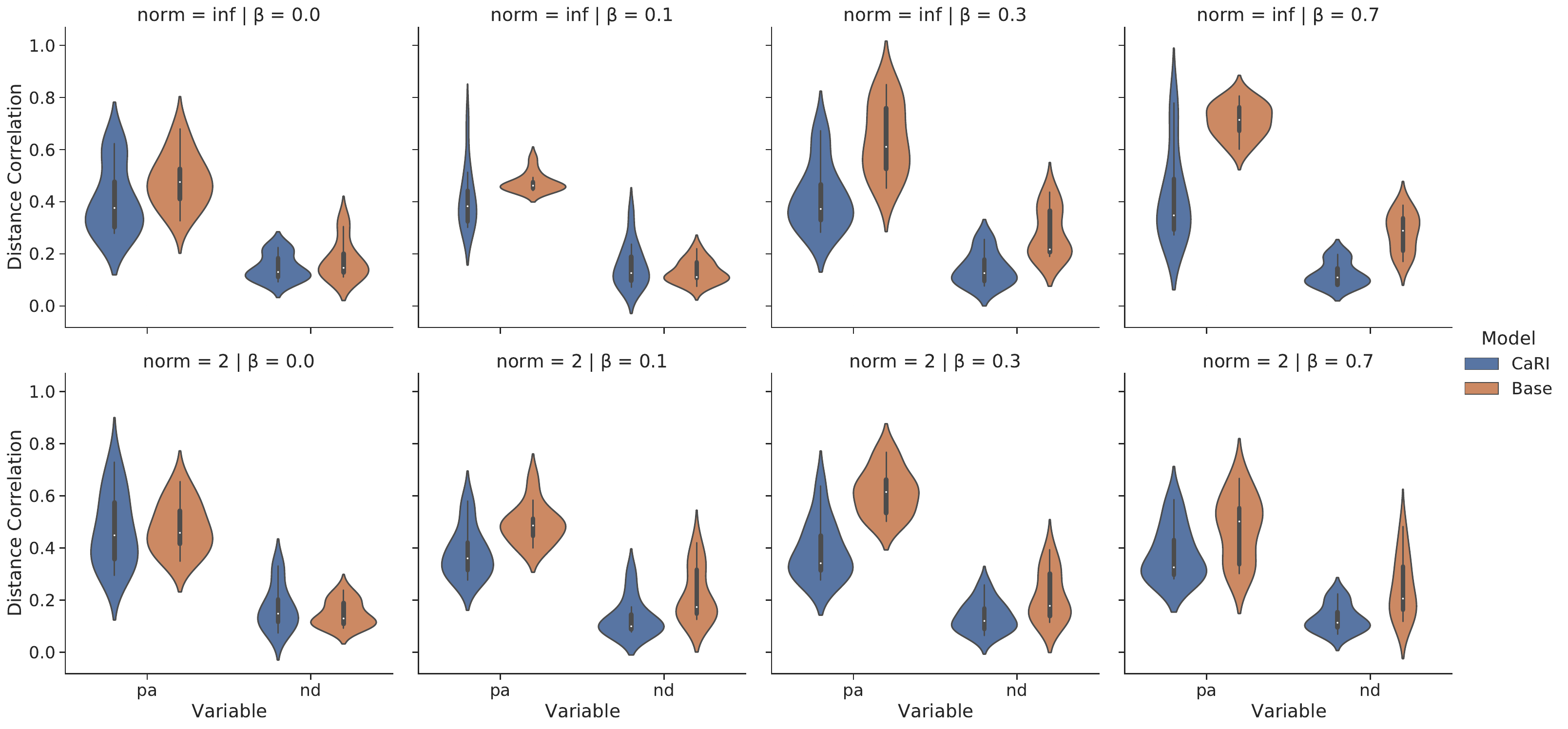}}
\caption{Identify results on CelebA-anno dataset over different range of $\beta$ during early optimization step.}
\label{fig:8}
\end{center}
\end{figure*}

\begin{figure*}[h]
\begin{center}
\centerline{\includegraphics[width=0.8\columnwidth]{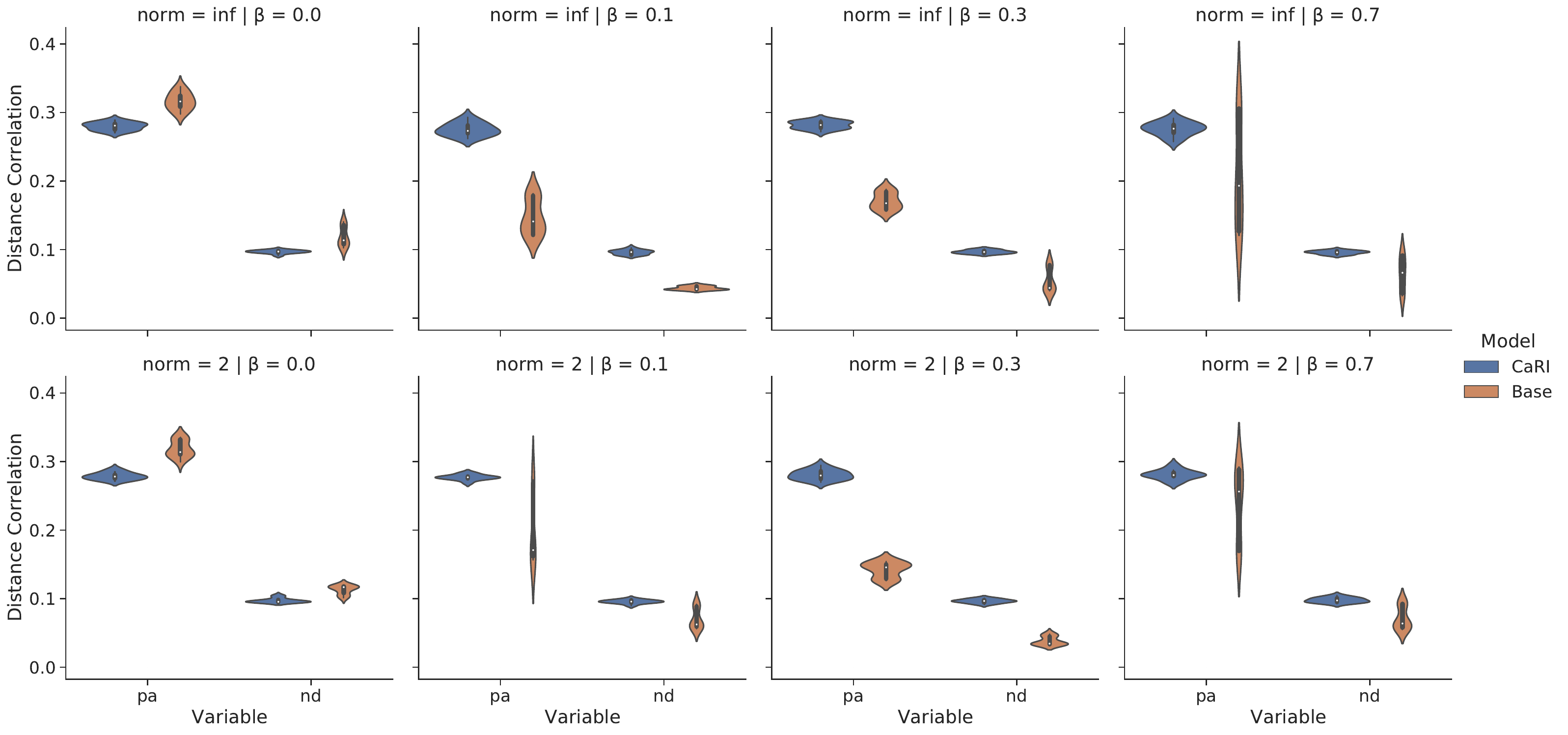}}
\caption{Identify results on CelebA-anno dataset over different range of $\beta$ after converging.}
\label{fig:9}
\end{center}
\end{figure*}

\subsection{Future Works} For future works, one promising direction is to involve the concept of Kolmogorov complexity in information theory. Different to mutual information and information entropy, Kolmogorov complexity is an asymmetric notion. Based on such a concept, we can develop a causal representation learning method without introducing an intervention network. Another direction is that our proposed method can be generalized to a mixture of anti-causal and causal learning frameworks where observation data contains both parents and descendants of outcome label $Y$. The information-theoretic-based sample complexity theorem can inspire the generalization error/risk analysis on causal representation learning and causal structure learning. 
Lastly, this paper is based on the assumption of the given causal graph Fig. \ref{fig:intro}. In the future, it is interesting to extend our method to more complex scenarios like sequential prediction, reinforcement learning etc. 

\end{document}